\documentclass[11pt]{article}

\usepackage[final]{acl}

\newcounter{observation}
\newcommand{\myobservation}[2]{
  \refstepcounter{observation}
  \label{#1}
  \noindent \textbf{Observation \theobservation.} \textit{\textbf{#2}}
}
\usepackage{algorithm}
\usepackage{algpseudocode}
\usepackage{times}
\usepackage{pifont}
\usepackage{latexsym}
\usepackage{subcaption}
\usepackage{booktabs}
\usepackage{tabularx}
\usepackage{adjustbox}
\usepackage{makecell}
\usepackage{amsmath}
\usepackage{amsfonts}
\usepackage{amsthm}

\newtheorem{lemma}{Lemma}

\usepackage{amssymb}
\usepackage{hhline}
\usepackage{multirow}
\usepackage{colortbl}
\usepackage[table]{xcolor}
\definecolor{lightblue}{rgb}{0.9, 0.9, 1.0} 
\definecolor{lightred}{rgb}{1.0, 0.9, 0.9}   
\definecolor{reddrop}{RGB}{204, 0, 0}    
\definecolor{greenup}{RGB}{0, 153, 51}   
\usepackage{url}

\usepackage[T1]{fontenc}

\usepackage[utf8]{inputenc}

\usepackage{microtype}

\usepackage{inconsolata}

\usepackage{graphicx}
\usepackage{hyperref}
\makeatletter
\newcommand{\nolinkfootnote}[1]{%
  \begingroup
  \let\@footnotemark\relax
  \let\@makefntext\noindent
  \protected@xdef\@thefnmark{}
  \footnotetext[0]{#1}
  \endgroup
}
\makeatother

\usepackage{enumitem}

\title{Sparse Growing Transformer: Training-Time Sparse Depth Allocation via Progressive Attention Looping}

\author{
\textbf{Yao Chen}$^{1,2*}$, 
\textbf{Yilong Chen}$^{1,2*}$, 
\textbf{Yinqi Yang}$^{3\ddagger}$, 
\textbf{Junyuan Shang}$^{3}$, 
\textbf{Zhenyu Zhang}$^{3}$, \\
\textbf{Zefeng Zhang}$^{1,2}$, 
\textbf{Shuaiyi Nie}$^{1,2}$, 
\textbf{Shuohuan Wang}$^{3}$, \\
\textbf{Yu Sun}$^{3}$, 
\textbf{Hua Wu}$^{3}$,
\textbf{HaiFeng Wang}$^{3}$,
\textbf{Tingwen Liu}$^{1,2\dagger}$ \\
$^1$ Institute of Information Engineering, Chinese Academy of Sciences \\
$^2$ School of Cyber Security, University of Chinese Academy of Sciences \\
$^3$ Baidu Inc. \\[3pt]
\small \texttt{\{chenyao2023, chenyilong, liutingwen\}@iie.ac.cn}
}

\begin{document}
\maketitle

~\nolinkfootnote{\textsuperscript{*} Equal contribution. 
\textsuperscript{\textdagger} Corresponding author. 
\textsuperscript{\textdaggerdbl} Project lead.
Our code is available at \url{https://github.com/YaoChen0203/Sparse-Growing-Transformer}.}

\begin{abstract}
Existing approaches to increasing the effective depth of Transformers predominantly rely on parameter reuse, extending computation through recursive execution.
Under this paradigm, the network structure remains static along the training timeline, and additional computational depth is uniformly assigned to entire blocks at the parameter level.
This rigidity across training time and parameter space leads to substantial computational redundancy during training.
In contrast, we argue that depth allocation during training should not be a static preset, but rather a progressively growing structural process. 
Our systematic analysis reveals a deep-to-shallow maturation trajectory across layers, where high-entropy attention heads play a crucial role in semantic integration. 
Motivated by this observation, we introduce the Sparse Growing Transformer (SGT).
SGT is a training-time sparse depth allocation framework that progressively extends recurrence from deeper to shallower layers via targeted attention looping on informative heads. 
This mechanism induces structural sparsity by selectively increasing depth only for a small subset of parameters as training evolves.
Extensive experiments across multiple parameter scales demonstrate that SGT consistently outperforms training-time static block-level looping baselines under comparable settings, while reducing the additional training FLOPs overhead from approximately $16$--$20\%$ to only $1$--$3\%$ relative to a standard Transformer backbone.
\end{abstract}

\section{Introduction}
Large language models (LLMs) have demonstrated exceptional knowledge integration and reasoning capabilities~\cite{DBLP:journals/corr/abs-2303-08774,zhu2024survey}. 
This success is fundamentally underpinned by the scalability of the Transformer architecture, which supports substantial expansion in both model width and, more crucially, network depth~\cite{vaswani2017attention,kaplan2020scaling,DBLP:journals/corr/abs-2203-15556}. 
For instance, leading models such as the LLaMA 3 and Qwen 3 series exhibit considerable depth across different parameter scales, ranging roughly from 30 to 120 layers~\cite{DBLP:journals/corr/abs-2407-21783,DBLP:journals/corr/abs-2505-09388}. 
Such depth is essential as it enables the learning of high-level hierarchical representations and is indispensable for the complex abstraction and compositional reasoning required for challenging tasks~\cite{tenney2019bert,chen2024can}. 
However, this increasing depth comes with a cost: stacking more layers inevitably expands the parameter count and memory footprint, leading to greater GPU memory consumption and deployment latency~\cite{hsieh2023distilling}.

To decouple model depth from parameter scale, recent studies have attempted to introduce recurrence into Transformer architectures, extending the effective computational depth to enable multi-step latent reasoning without inflating parameters~\cite{fan2024looped,chen-etal-2025-inner,geiping2025scaling}. 
However, existing recursive architectures adopt a training-time static topology and uniformly apply block-level reuse at the parameter level, lacking fine-grained differentiation of parameter roles (see Appendix~\ref{app: related} for a detailed related work discussion). 
This rigidity across training time and parameter space leads to a substantial and unnecessary increase in training computational cost.
Such coarse-grained recurrence contrasts with biological neural circuits, where autaptic self-connections emerge selectively within specific neuronal populations as evolutionarily and developmentally specialized microcircuit motifs, rather than being uniformly instantiated across all neurons~\cite{bacci2003functional,jiang2015developmental,pan2025functional}.
Guided by this perspective, we introduce a new paradigm of \textbf{\textit{Training-Time Structural Sparsity}}, in which model depth progressively self-grows during training through selective allocation of additional computation to fine-grained subsets of parameters.

To instantiate this paradigm during training, we require a reliable signal to guide structural growth. 
Recent studies have revealed a clear division of labor within Transformers: Self-Attention Modules drive dynamic in-context reasoning, while Feed-Forward Networks predominantly encode static statistical regularities~\cite{geva2021transformer,DBLP:journals/corr/abs-2209-11895,chen2025distributionalassociationsvsincontext}. 
Attention entropy, a metric quantifying the reasoning uncertainty and information content of attention heads, has been shown to correlate closely with contextual integration capabilities~\cite{zhang2025attention}. 
Naturally, attention entropy can serve as a critical indicator.
We further investigate the functional characteristics and training dynamics of attention entropy. 
Our analysis reveals two key findings: (1) Functionally, high-entropy heads act as pivotal hubs for semantic integration rather than noise; (2) Dynamically, layers follow a \textit{deep-to-shallow} maturation trajectory, where deeper layers differentiate earlier in training while shallower layers evolve more gradually.
Guided by these findings, we propose the \textbf{Sparse Growing Transformer (SGT)}, a self-growing depth architecture via progressive attention looping that realizes training-time structural sparsity. 
Specifically, SGT establishes a progressive evolution aligned with the maturation trajectory: recurrence is first activated in early-differentiating deeper layers and then gradually extended toward shallower ones, while within active layers, recursive looping is selectively applied to high-entropy heads to concentrate computation on critical semantic units.

Our contributions are summarized as follows:
\textbf{1)} We conduct a systematic analysis of attention entropy from both functional and dynamic perspectives. Our findings show that high-entropy heads serve as pivotal hubs for semantic integration, and that layers follow a depth-dependent evolutionary trajectory characterized by a \textit{deep-to-shallow} maturation pattern.
\textbf{2)} We formalize the paradigm of \textbf{\textit{Training-Time Structural Sparsity}} and instantiate it through the \textbf{Sparse Growing Transformer (SGT)}, which progressively allocates recurrent computation to high-entropy heads via fine-grained structural growth. We further provide a theoretical analysis demonstrating that concentrating recurrence on high-entropy components accelerates convergence.
\textbf{3)} Extensive experiments across multiple parameter scales demonstrate that SGT consistently outperforms training-time static block-level looping baselines under comparable settings, while reducing the additional training FLOPs overhead from approximately $16$--$20\%$ to only $1$--$3\%$ relative to a standard Transformer backbone.

\section{Preliminaries}
\label{sec: preliminaries}

To facilitate the understanding of the empirical observations in Section~\ref{Observations} and the formalization of our proposed method in Section~\ref{Method}, we establish the notation for the standard attention mechanism and introduce the definition of attention entropy.

\paragraph{Multi-Head Attention.}
We formulate the input hidden states as $H \in \mathbb{R}^{N \times d}$, where $N$ is the sequence length and $d$ is the model dimension. 
We focus on the Multi-Head Attention (MHA) module, which processes $H$ through multiple parallel attention heads.
For the $i$-th head, the input is projected into queries $Q^{(i)}$, keys $K^{(i)}$, and values $V^{(i)}$ using parameter matrices $W_Q^{(i)}, W_K^{(i)}, W_V^{(i)} \in \mathbb{R}^{d \times d_h}$, where $d_h$ is the head dimension. The attention matrix $A^{(i)} \in \mathbb{R}^{N \times N}$ is computed as:
\begin{equation}
    A^{(i)} = \operatorname{softmax}\left(\frac{Q^{(i)} (K^{(i)})^\top}{\sqrt{d_h}}\right).
    \label{eq: attention_score}
\end{equation}
The output of a single head is then given by projecting the weighted values:
\begin{equation}
    \operatorname{Attn}^{(i)}(H) = (A^{(i)} H W_V^{(i)}) W_O^{(i)},
\end{equation}
where $W_O^{(i)} \in \mathbb{R}^{d_h \times d}$ represents the output projection matrix of the $i$-th head. 
The outputs from all parallel heads are summed and then added to the input $H$ via a residual connection to produce the post-attention hidden state $H_{\text{post}}$:
\begin{equation}
    H_{\text{post}} = H + \sum_{i=1}^{N_{head}} \operatorname{Attn}^{(i)}(H),
\end{equation}
where $N_{head}$ denotes the number of attention heads per layer.
This resulting representation $H_{\text{post}}$ is subsequently processed by the Feed-Forward Network.

\paragraph{Attention Entropy.}
Attention entropy serves as a metric for quantifying the concentration and uncertainty of information flow. 
Considering computational efficiency and the fact that the final token aggregates information from the entire preceding context (Appendix~\ref{app: Attention Entropy Computation and Rationale}), we focus on the attention entropy of the final token.
To eliminate the bias introduced by sequence lengths, we utilize the \textit{length-normalized attention entropy}. Formally, for the $i$-th head, this is defined as:
\begin{equation}
    \mathcal{E}^{(i)} = -\frac{1}{\log N} \sum_{j=1}^{N} A^{(i)}_{N, j} \log A^{(i)}_{N, j},
    \label{eq: normalized-entropy}
\end{equation}
where $\log N$ serves as the length normalization term, and $A^{(i)}_{N, j}$ denotes the attention weight in the $i$-th head assigned by the $N$-th token to the $j$-th input token, with the resulting metric $\mathcal{E}^{(i)}$ lying in the interva $[0, 1]$.
Based on this, we further define the \textit{layer-wise attention entropy} to capture the uncertainty of the $l$-th layer. 
This is computed by averaging the normalized entropies across all $N_{head}$ heads within the $l$-th layer:
\begin{equation}
    \bar{\mathcal{E}}^{(l)} = \frac{1}{N_{head}} \sum_{i=1}^{N_{head}} \mathcal{E}^{(l, i)}.
    \label{eq:layer_mean_entropy}
\end{equation}
In the following section, we use $\mathcal{E}^{(i)}$ and $\bar{\mathcal{E}}^{(l)}$ as the analytical lens to investigate attention entropy, providing the insights for the design of SGT.

\section{Observations}
\label{Observations}
By visualizing attention entropy across representative open-source models, we observe that \textbf{\textit{attention entropy is closely correlated with the model architecture}}.
Specifically, this correlation manifests as a universal \textit{three-phase} progression at the layer level (Figure~\textbf{\ref{fig: layer_atten_entropy}}) and highly consistent head-level patterns across varying parameter scales with identical layer-head configurations (Figure~\textbf{\ref{fig: Layer_Head_Attention_Entropy_Patterns}}), establishing attention entropy as a stable architectural attribute.
Building upon this foundation, we further investigate the functional characteristics and training dynamics of attention entropy.

\begin{figure}[!t]
    \centering
    \includegraphics[width=\linewidth]{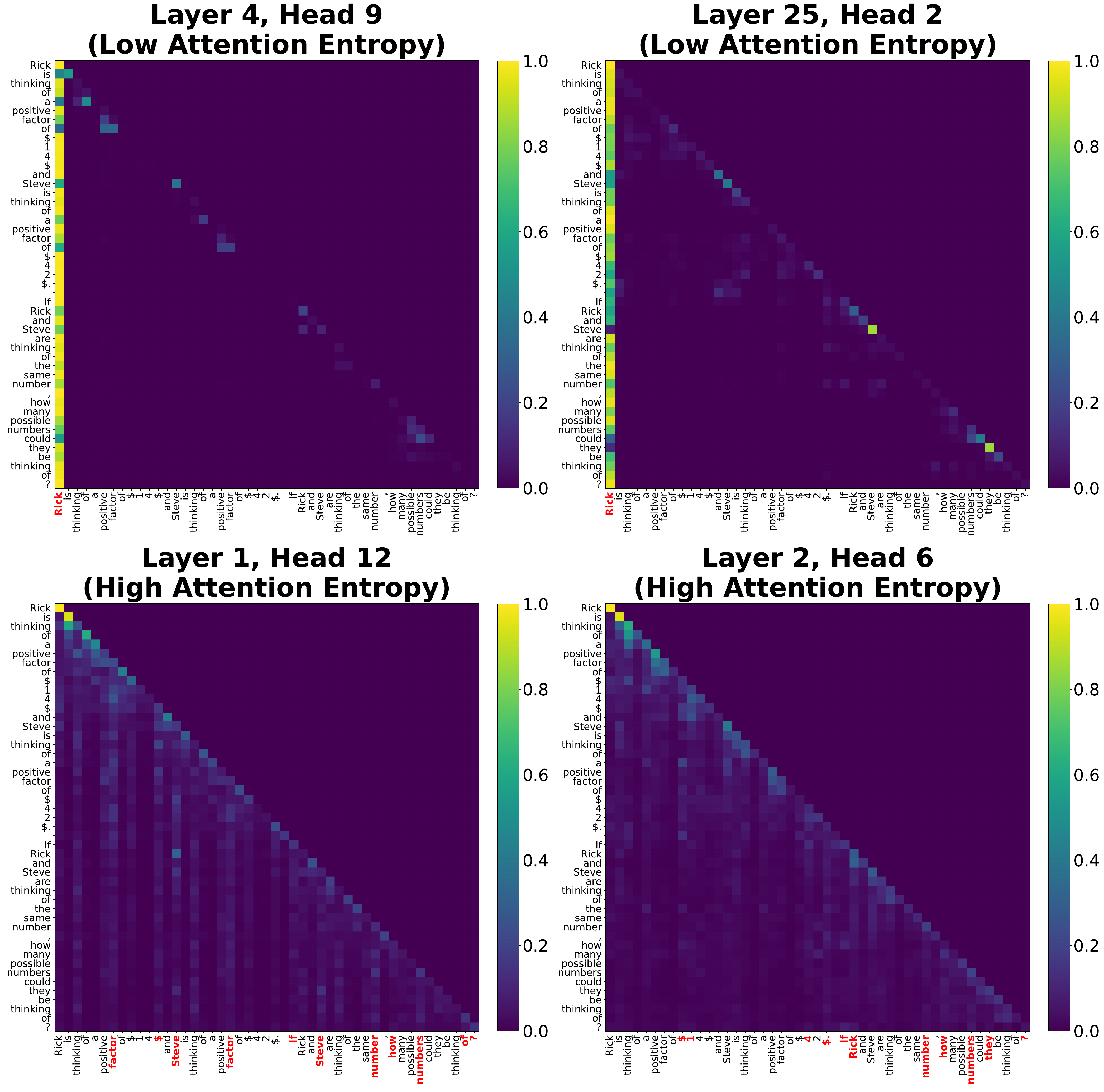} 
        \caption{Visualization of attention heatmap in low- and high-entropy heads from Qwen3-0.6B. Along the horizontal axis, tokens highlighted in red denote the subset that receives the top 50\% of the attention from the final query position (details in Appendix ~\ref{app: Qualitative Analysis}).}
        \label{fig: sample_visual_gsm8k}
\end{figure}

\begin{figure}[!t]
    \centering
    \includegraphics[width=\linewidth]{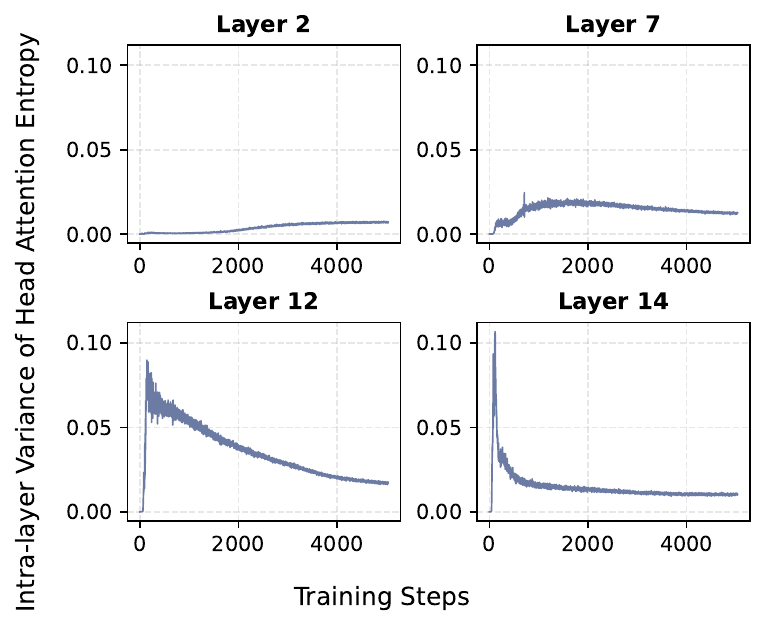} 
        \caption{Evolution of Intra-layer Variance in Head-wise Attention Entropy Across Training Steps for Models Pretrained from Scratch (we visualize four representative layers; see Figure~\ref{fig: training/head_variance_unified} for detailed results)}
        \label{fig: head_variance_selected_layers}
\end{figure}

\vspace{0.5em} 

\myobservation{ob:functional_dichotomy}{High-entropy heads serve as information-dense hubs critical for global reasoning.}
We identified heads in the Qwen3-0.6B model that consistently exhibit high or low entropy across 11 diverse tasks, and then examined these groups (details in Appendix~\textbf{\ref{app: Entropy-based Head Selection}}).
Qualitatively, low-entropy heads predominantly exhibit attention sink behavior, focusing on local tokens. In contrast, high-entropy heads distribute attention selectively across globally distributed tokens rather than uniformly, capturing critical evidence for reasoning (Figure\textbf{~\ref{fig: sample_visual_gsm8k},~\ref{fig: sample_visual_math500}} and\textbf{~\ref{fig: sample_visual_mmlu_college_physics}}).
Quantitatively, head masking experiments demonstrate that disabling these high-entropy heads results in the most substantial performance degradation (Table~\textbf{\ref{tab: mask_results}}).
Based on these findings, we select high-entropy heads as the targets for allocating additional computational depth to facilitate iterative \textit{inner thinking}.

\vspace{0.5em} 

\myobservation{ob:maturation_trajectory}{Attention entropy reveals a deep-to-shallow maturation trajectory: deeper layers differentiate early in training, while shallower layers evolve more gradually.}
To investigate the dynamic properties of attention entropy, we train a 573M LLaMA model from scratch and monitor head-wise attention entropy across all layers (Figure~\textbf{\ref{fig: training/head_entropy}}). 
We observe a general downward trend in entropy accompanied by significant discrepancy, which we quantify using intra-layer entropy variance (Figure~\textbf{\ref{fig: head_variance_selected_layers}} and Figure~\textbf{\ref{fig: training/head_variance_unified}}). 
Specifically, for intra-layer entropy variance, deep layers exhibit a rapid surge followed by a decline, whereas shallow layers show a steady increase. 
Interpreting such variance as a measure of functional differentiation (Observation~\ref{ob:functional_dichotomy}), we find that deep-layer heads differentiate earlier, while shallow-layer heads evolve over a longer horizon. These depth-dependent evolutionary patterns serve as the primary motivation for our progressive growth strategy.

\begin{figure}[!t]
    \centering
    \includegraphics[width=\linewidth]{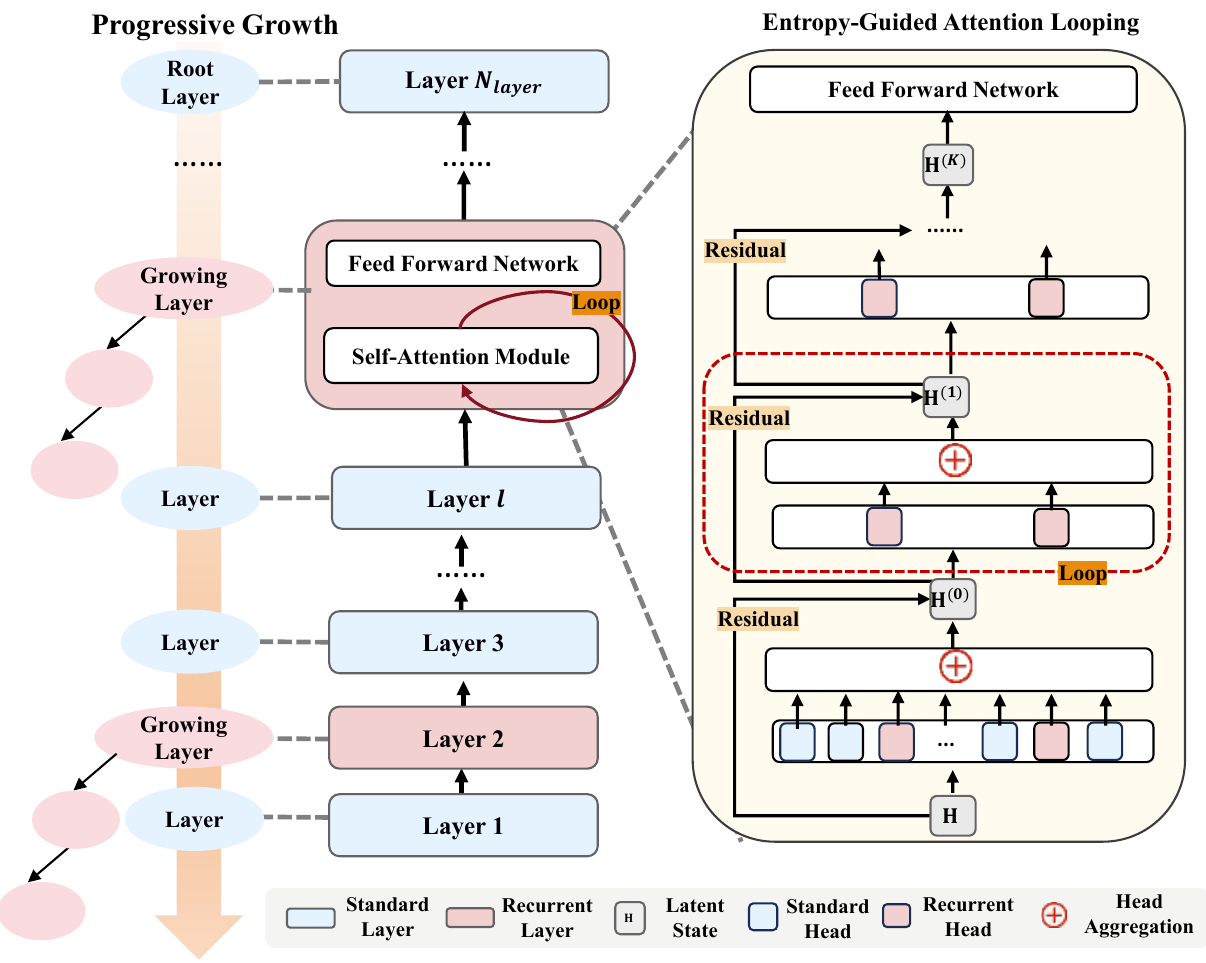} 
       \caption{The Sparse Growing Transformer (SGT) architecture. It implements Training-Phase Structural Sparsity via a deep-to-shallow progressive growth schedule (Left) and selective high-entropy head looping (Right).}
        \label{fig: method}
\end{figure}

\section{Method}
\label{Method}
Building upon our above observations, we propose the Sparse Growing Transformer (SGT), a self-growing depth architecture via progressive attention looping (Figure~\ref{fig: method}). 
Specifically, SGT progressively grows fine-grained recurrent micro-circuits during training, selectively allocating additional depth to high-value parameters within a block rather than uniformly reapplying entire layers. 
This growth unfolds along the \textit{deep-to-shallow} maturation trajectory and concentrates computation on a subset of functional components, thereby inducing training-time structural sparsity and reducing substantial pretraining overhead.
This section details the SGT mechanism composed of \textit{Entropy-Guided Attention Looping} and \textit{Progressive Growth Training}.

\subsection{Entropy-Guided Attention Looping}
\label{Entropy-Guided Attention Looping}
Guided by our Observation~\textbf{\ref{ob:functional_dichotomy}} that high-entropy heads serve as information-dense hubs critical for global reasoning, we identify them as the optimal candidates for allocating additional computational depth, rather than indiscriminately applying recurrence to entire blocks.
Formally, following the notation in Section~\ref{sec: preliminaries}, we denote the post-attention hidden state $H_{\text{post}}$ as $H^{(0)}$, which serves as the initial state for the looping mechanism. 
Subsequently, we apply a selective head-level recursive loop defined by an efficient \textit{allocation operator} $\mathcal{A}$, which maps the layer's attention entropy distribution to a specific set of active heads $\mathcal{S}$ for optimized information flow. 
Let $\mathcal{S}$ be the set of active heads designated for looping, and let $k$ denote the $k$-th loop step. 
For step $k = 1,\dots,K$, the recursive update is formulated as:
\begin{equation}
    H^{(k)} = H^{(k-1)} + \sum_{i \in \mathcal{S}} \operatorname{Attn}^{(i)}(H^{(k-1)}).
    \label{eq:sga_update}
\end{equation}

Specifically, for a given layer, we compute the length-normalized attention entropy $\mathcal{E}^{(i)}$ (Eq.~\ref{eq: normalized-entropy}) for each head $i$. 
The allocation operator $\mathcal{A}$ then identifies the top-$h$ heads exhibiting the highest entropy to form the active recursive set $\mathcal{S}$:
\begin{equation}
    \mathcal{S} = \mathcal{A}(\{\mathcal{E}^{(i)}\}) = \operatorname*{arg\,top}_{h} \left( \{\mathcal{E}^{(i)}\}_{i=1}^{N_{head}} \right).
    \label{eq:head_selection}
\end{equation}

Upon completion of the grown depth $K$, the final state $H^{(K)}$ is passed to the Feed-Forward Network. 
Notably, we provide a theoretical justification in Appendix~\textbf{\ref{app: Theoretical Analysis}}, demonstrating that applying recursive processing to high-entropy heads accelerates convergence toward the quasi-stationary state in recursive dynamics. 
By concentrating computational resources exclusively on these high-uncertainty units, this mechanism enables the SGT to form fine-grained recurrent micro-circuits. 
The emergence and expansion of these structures along the training process are governed by a \textit{Progressive Growth Training} strategy.

\subsection{Progressive Growth Training}
\label{Progressive Growth Training}
To operationalize this \textit{Progressive Growth Training} strategy, we design a progressive growth schedule that aligns structural evolution with the model’s intrinsic maturation trajectory (Observation~\textbf{\ref{ob:maturation_trajectory}}). 
By activating recurrent micro-circuits gradually over the training timeline rather than from the outset, this design induces temporal sparsity during training and substantially reduces pretraining computational overhead.

This progressive growth mechanism is formalized through a growth operator $\mathcal{O}$ that updates the looping configuration along training time.
At training step $t$, the architectural state is defined as
\begin{equation}
\Theta_t := \{\mathcal{L}_t,\; l_t^*,\; \{K_l\}_{l \in \mathcal{L}_t} \},
\label{eq:arch_state_revised}
\end{equation}
where $\mathcal{L}_t$ is the set of active looping layers, $l_t^*$ denotes the currently growing layer, $K_l$ is the current loop depth for layer $l \in \mathcal{L}_t$.
The growth operator $\mathcal{O}$ updates $\Theta_t$ based on the attention entropy $\bar{\mathcal{E}}$
\begin{equation}
\Theta_{t+\Delta t} = \mathcal{O}(t, \bar{\mathcal{E}}; \Theta_t).
\end{equation}

Initially, during the \textit{Warm-up Phase} ($t < t_{\text{start}}$), 
the model trains as a standard Transformer without structural growth. 
The warm-up period allows attention entropy to stabilize, providing a reliable signal for subsequent structural decisions.
Upon entering the \textit{Growing Phase} ($t \ge t_{\text{start}}$), the growth operator $\mathcal{O}$ periodically updates the architectural state $\Theta_t$ every $\Delta t$ steps.
At each growth step, all layers are ranked to form a candidate pool $\mathcal{C}$ according to their mean attention entropy
\begin{equation}
\mathcal{C} = \operatorname*{arg\,top}_{L}\!\left(\{\bar{\mathcal{E}}^{(l)}\}\right).
\end{equation}
Following the \textit{deep-to-shallow} maturation trajectory (Observation~\ref{ob:maturation_trajectory}), structural expansion proceeds in an ordered manner. 
If the currently growing layer $l^*$ remains among the candidate pool $\mathcal{C}$ and has not reached the maximum depth $K_{\max}$, its loop depth is updated as $K_{l^*} \leftarrow K_{l^*} + 1$.
Otherwise, provided that the number of active looping layers has not yet reached $L$, a new layer is activated. 
The newly selected layer is the deepest eligible candidate not yet in $\mathcal{L}_t$, subject to the constraint that it lies deeper than the shallowest active looping layer:
\begin{equation}
l^* = \max \{\, l \mid l \in \mathcal{C},\; l \notin \mathcal{L}_t,\; l < \min(\mathcal{L}_t) \,\}.
\end{equation}
This rule enforces a monotonic deep-to-shallow expansion, anchoring growth in earlier-stabilized deeper layers before extending toward shallower ones, preserving training stability during structural expansion.

Once the target configuration is reached, the training enters the \textit{Fixed Phase}, in which the architecture is frozen to allow full parameter convergence.
The complete training procedure is provided in \textbf{Algorithm~\ref{alg:SGT}}.
Through this staged evolution, computational depth becomes a progressively expanding yet temporally sparse structure, enabling SGT to gradually form fine-grained recurrent micro-circuits aligned with intrinsic entropy dynamics.

\begin{table*}[!t]
\centering
\small
\definecolor{ref_blue}{RGB}{230, 242, 255} 
\definecolor{ref_red}{RGB}{255, 235, 235}
 \definecolor{ref_yellow}{RGB}{255, 250, 230}
\renewcommand{\arraystretch}{1.3}
\setlength{\tabcolsep}{4.5pt}
\begin{tabular}{l|c|ccccccccc}
\toprule
\multirow{2}{*}{\textbf{Model}} & \multirow{2}{*}{\textbf{FLOPs}} & \multicolumn{9}{c}{\textbf{Reasoning \& Knowledge} ($\uparrow$)} \\
\cmidrule(lr){3-10}
 & & \textbf{ARC-E} & \textbf{WG} & \textbf{SIQA} & \textbf{Hella.} & \textbf{OBQA} & \textbf{CSQA} & \textbf{BA} & \textbf{MMLU} & \textbf{Avg} \\
\midrule
\rowcolor{ref_yellow}
\textbf{Vanilla (275M)} &42.83 &44.56  &49.57 &41.30 &30.89 &27.40 &29.24 &22.70 &24.80 &33.81 \\
\rowcolor{ref_blue}
\textbf{Block Loop} &51.50 (\textcolor{greenup}{+20.24\%}) &44.74  &50.36  &41.10 &\textbf{31.85} &27.60 &\textbf{29.48} &21.80 &25.38 &34.04 \\
\rowcolor{ref_red}
\textbf{SGT(K=1)} &43.46 (\textcolor{greenup}{1.47\%}) &45.09  &51.46 &41.40 &31.21 &\textbf{28.00} &29.16 &24.07 &25.19 &34.48 \\
\rowcolor{ref_red}
\textbf{SGT(K=2)} &43.88 (\textcolor{greenup}{+2.45\%}) &45.79  &\textbf{52.33} &\textbf{42.38} &31.41 &26.80 &28.75 &\textbf{24.47} &25.45 &\textbf{34.64}\\
\rowcolor{ref_red}
\textbf{SGT(K=3)} &44.21 (\textcolor{greenup}{+3.22\%}) &\textbf{46.32}  &50.91 &41.81 &31.43 &27.40 &28.83 &22.37 &\textbf{26.16} &34.40 \\
\midrule 
\rowcolor{ref_yellow}
\textbf{Vanilla (573M)} & 87.41 &47.01  &50.35   &41.86  &35.24  &\textbf{29.00}  &29.24  &24.03  &26.37  &35.39 \\
\rowcolor{ref_blue}
\textbf{Block Loop} &101.76 (\textcolor{greenup}{+16.42\%})  &48.07  &50.19   &42.27  &\textbf{36.11} &28.00  &30.30  &24.90  &26.68  &35.82    \\
\rowcolor{ref_red}
\textbf{SGT(K=1)} &88.15 (\textcolor{greenup}{+0.85\%}) &45.61  &49.64   &42.22 &35.36 &\textbf{29.00} &\textbf{31.29} &24.93 &\textbf{27.42} &35.68 \\
\rowcolor{ref_red}
\textbf{SGT(K=2)} &88.60 (\textcolor{greenup}{+1.36\%}) &47.37  &50.20  &\textbf{42.47}  &35.42 &28.88  &28.50  &24.80  &27.10  &35.59  \\
\rowcolor{ref_red}
\textbf{SGT(K=3)} &88.96 (\textcolor{greenup}{+1.77\%}) &\textbf{50.00}  &\textbf{53.43} &42.22  &35.25 &27.40  &29.89  &\textbf{26.07}  &27.03  &\textbf{36.41}  \\
\midrule 
\rowcolor{ref_yellow}
\textbf{Vanilla (1.2B)} &175.76 &50.00  &50.19 &43.39 &\textbf{41.02} &30.00 &31.69 &\textbf{26.60} &\textbf{27.76} &37.58 \\
\rowcolor{ref_blue}
\textbf{Block Loop} &205.65 (\textcolor{greenup}{+17.01\%}) &50.17 &51.38 &42.42 &41.01 &31.00 &31.28 &25.83 &27.01 &37.51 \\
\rowcolor{ref_red}
\textbf{SGT(K=1)} &177.06 (\textcolor{greenup}{+0.74\%})  &50.88 &51.78  &\textbf{43.50} &40.30 &30.00 &31.36 &25.43 &27.65 &37.61 \\
\rowcolor{ref_red}
\textbf{SGT(K=2)} &177.85 (\textcolor{greenup}{+1.19\%}) &\textbf{51.05} &\textbf{52.49} &43.14 &40.62 & 31.00 &32.11 &25.33 &27.52 &\textbf{37.91}  \\
\rowcolor{ref_red}
\textbf{SGT(K=3)} &178.48 (\textcolor{greenup}{+1.55\%}) &49.30  &51.54 &43.04 &40.51 &\textbf{31.40} &\textbf{32.27} &25.80 &26.92 &37.60 \\
\bottomrule
\end{tabular}
\caption{Main experimental results. FLOPs denote the total training compute, measured in units of $10^{18}$. The best results at each model scale are shown in \textbf{bold}. $K$ denotes $K_{\max}$, the upper bound on the loop depth for the selected loop layers.}
\label{tab: comparative_results}
\end{table*}

\section{Experiments}

\subsection{Setup}
\paragraph{Data.}
C4~\cite{raffel2020exploring} is a large-scale, cleaned web text corpus widely used for language model pretraining. 
We train all model variants on a 20B-tokens subset of C4 and evaluate perplexity on a validation set constructed from the same corpus.

\paragraph{Training.}
All experiments are conducted using the OLMo open-source training framework~\cite{DBLP:conf/acl/GroeneveldBWBKT24}. 
We adopt the AdamW optimizer with a learning rate of \(6.0 \times 10^{-4}\) applied to all parameters. 
Models are trained with a sequence length of 4096 and a global batch size of 1024 for 5035 training steps. 
Training is performed on 8 NVIDIA A100 GPUs with 40 GB of memory each.
All model scales (Table~\ref{tab: model_detailed_configuration}) adopt the LLaMA-style architecture and are pre-trained from scratch under the same training setup to ensure fair comparison.

\paragraph{Evaluation.}
Our experimental evaluation assesses model capabilities across a diverse set of Reasoning \& Knowledge tasks.
Specifically, We report accuracy on ARC-Easy (ARC-E)~\cite{clark2018think}, SocialIQA (SIQA)~\cite{welbl2017crowdsourcing}, OpenBookQA (OBQA)~\cite{OpenBookQA2018}, SCIQ~\cite{welbl2017crowdsourcingmultiplechoicescience}, WinoGrande (WG)~\cite{sakaguchi2021winogrande}, BasicArithmetic (Basic Arith)~\cite{brown2020languagemodelsfewshotlearners}, CommonsenseQA (CSQA)~\cite{talmor-etal-2019-commonsenseqa}, HellaSwag (Hella.)~\cite{zellers2019hellaswag}, and the four subtasks of MMLU~\cite{hendryckstest2021} (STEM, Social Sciences, Humanities, and Other).

\paragraph{Baseline.}
We compare the following methods across varying model scales. 
1) \emph{Vanilla}: The standard Transformer architecture without any recurrence mechanism, serving as the reference baseline.
2) \emph{Block Loop}: A baseline implementing block-level recurrence by iteratively reusing entire Transformer blocks.
3) Sparse Growing Transformer (SGT): Our proposed self-growing depth architecture utilizing entropy-guided attention looping. We configure it with $h=2,L=3$ and evaluate three variants with $K_{\text{max}} \in \{1, 2, 3\}$.
To ensure a fair comparison, for \emph{Block Loop}, we select the three layers exhibiting the highest entropy from a vanilla model pre-trained for $t_{\text{start}}$ steps as the fixed looping layers (more details in Appendix~\ref{app: Experimental Details and Supplementary Analysis}).

\begin{figure}[!t]
    \centering
    \includegraphics[width=\linewidth]{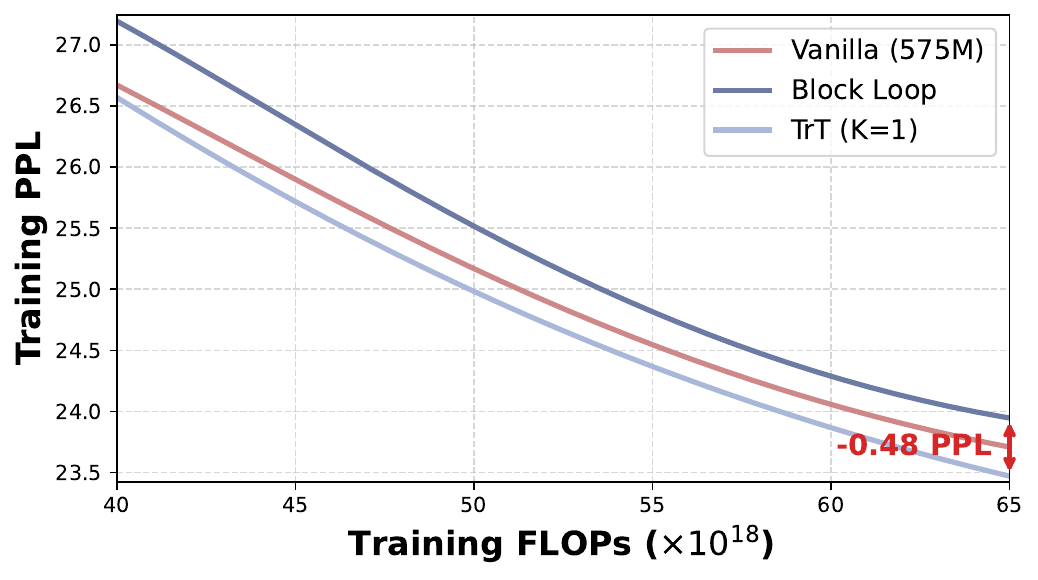} 
        \caption{Comparative convergence trajectories of training perplexity (PPL) versus cumulative training FLOPs for the 573M model scale.}
    \label{fig: train_ppl}
\end{figure}

\subsection{Main Results}

\paragraph{Improvements on Reasoning Tasks.}
Notably, SGT achieves particularly pronounced improvements on reasoning-intensive benchmarks. On tasks such as WG, ARC-E, and CSQA, which require complex relational integration and multi-step inference, SGT demonstrates substantial gains of 3.24, 1.93, and 0.99 points, respectively, over the Block Loop baseline at the 573M scale. Crucially, SGT achieves these superior results with significantly higher efficiency: while Block Loop incurs a heavy +16.42\% computational overhead, our method ($K=3$) surpasses Block Loop's average performance by 0.59 points while increasing FLOPs by only +1.77\%, which is approximately an order of magnitude less additional cost. These results corroborate that high-entropy heads play a pivotal role in semantic dependency construction, and that selectively deepening their computation enables more thorough inner thinking without the burden of redundant parameters.

\paragraph{Training Efficiency and Convergence.}
We evaluate training efficiency by comparing the convergence trajectories of different strategies relative to their computational consumption. As illustrated in Figure~\ref{fig: train_ppl}, when aligned by equivalent training FLOPs, SGT consistently maintains a lower training perplexity compared to both the Vanilla and Block Loop baselines. Specifically, at a cumulative compute of $65 \times 10^{18}$ FLOPs, SGT achieves a PPL that is 0.48 points lower than the Vanilla baseline. In contrast, the Block Loop method exhibits higher perplexity than Vanilla at the same FLOP budget, indicating computational inefficiency. This demonstrates that by strategically allocating depth to high-entropy components, SGT not only improves final performance but also significantly accelerates convergence efficiency per unit of compute.

\subsection{Ablation Studies}
\label{sec: ablation}
\subsubsection{Impact of Loop Components}
To evaluate the specific contributions of distinct model components to iterative refinement, we conduct ablation experiments on the 573M model by applying loop operations at varying depths (Layer 1, 7, and 15). 
We compare four configurations: Block Loop, All Head Loop, High-Entropy Head Loop (top-2 highest), and Low-Entropy Head Loop (top-2 lowest). 
In all configurations, the loop operation is restricted to a single iteration on the designated layer to isolate component effects. Further details are provided in Appendix~\ref{app: Details of Ablation Experiments}.

\paragraph{Targeted allocation of computational depth to high-entropy components yields the optimal efficiency-performance ratio.}
As shown in Table~\ref{tab: ablation}, the High-Entropy Head Loop consistently achieves the most significant PPL reduction across all tested depths while incurring the smallest increase in FLOPs (+0.38\%). Specifically, within the attention mechanism, looping the top-2 high-entropy heads significantly outperforms both looping all heads and looping low-entropy heads. These empirical findings are further corroborated by our theoretical analysis in Appendix~\ref{app: Theoretical Analysis}, which establishes that recursive dynamics on high-entropy matrices exhibit faster convergence due to their superior token-mixing properties.

\begin{table}[!t]
\centering
\small
\definecolor{darkgreen}{rgb}{0.0, 0.5, 0.0}
\renewcommand{\arraystretch}{1.2}
\newcommand{\diff}[1]{\textcolor{darkgreen}{ (#1)}}
\setlength{\tabcolsep}{6pt}
\begin{tabular}{l c c}
\toprule
\textbf{Loop Component} & \textbf{FLOPs} ($\times 10^{18}$) & \textbf{Perplexity} $\downarrow$ \\
\midrule
Vanilla & 87.41 & 23.973 \\
\midrule
\multicolumn{3}{c}{\textbf{Layer 2}} \\
\midrule
Block Loop & 92.20 \diff{+5.48\%} & 23.560 \diff{-0.413} \\
Attention Loop  & 89.01 \diff{+1.83\%} & 23.503 \diff{-0.470} \\
High-Ent. Loop & 87.74 \diff{+0.38\%} & \textbf{23.452} \diff{-0.521} \\
Low-Ent. Loop & 87.74 \diff{+0.38\%} & 23.502 \diff{-0.471} \\
\midrule
\multicolumn{3}{c}{\textbf{Layer 8}} \\
\midrule
Block Loop & 92.20 \diff{+5.48\%} & 23.836 \diff{-0.137} \\
Attention Loop & 89.01 \diff{+1.83\%} & 23.537 \diff{-0.436} \\
High-Ent. Loop & 87.74 \diff{+0.38\%} & \textbf{23.487} \diff{-0.486} \\
Low-Ent. Loop & 87.74 \diff{+0.38\%} & 23.505 \diff{-0.468} \\
\midrule
\multicolumn{3}{c}{\textbf{Layer 16}} \\
\midrule
Block Loop & 92.20 \diff{+5.48\%} & 23.711 \diff{-0.262} \\
Attention Loop & 89.01 \diff{+1.83\%} & 23.711 \diff{-0.262} \\
High-Ent. Loop & 87.74 \diff{+0.38\%} & \textbf{23.503} \diff{-0.470} \\
Low-Ent. Loop & 87.74 \diff{+0.38\%} & 23.797 \diff{-0.176} \\
\bottomrule
\end{tabular}
\caption{Ablation study of different loop components at varying layer depths. We report total training FLOPs (in units of $10^{18}$) and perplexity (PPL) on the validation set, where lower is better. High-Ent. Loop denotes looping the top-2 highest-entropy heads, while Low-Ent. Loop denotes looping the top-2 lowest-entropy heads.}
\label{tab: ablation}
\end{table}

\paragraph{Block-level recurrence incurs computational redundancy due to a lack of fine-grained parameter discrimination.}
Our results indicate that attention-level looping consistently outperforms or matches block-level looping regardless of layer depth. This empirical finding supports our hypothesis that block-level approaches lack deep insights into the granular roles of distinct parameters. By treating the Transformer block as a monolithic unit, these methods fail to isolate the components truly effective for iterative refinement. Indiscriminately applying recurrence to the entire block results in significant computational redundancy without yielding proportional performance gains.

\begin{table}[!t]
\centering
\small
\renewcommand{\arraystretch}{1.3}
\setlength{\tabcolsep}{3.7pt}
\begin{tabular}{lccccc}
\toprule
\textbf{Method} & CSQA & SIQA & SCIQ & $\text{MMLU}_\text{STEM}$ & Avg \\
\midrule
Vanilla & 29.24 & 41.86 & 71.00 & 26.28 & 42.10 \\
S2D & 30.39 & 41.40 & 71.80 & 27.40 & 42.75 \\
D2S (Ours) & \textbf{31.29} & \textbf{42.22} & \textbf{72.00} & \textbf{28.46} & \textbf{43.49} \\
\bottomrule
\end{tabular}
\caption{Ablation study on growth direction. Experiments are conducted on the 573M model with SGT ($K=1$) configuration. \textbf{D2S} denotes deep-to-shallow growth, while \textbf{S2D} represents shallow-to-deep growth.}
\label{tab: growth_direction}
\end{table}

\begin{figure}[!t]
    \centering
    \includegraphics[width=\linewidth]{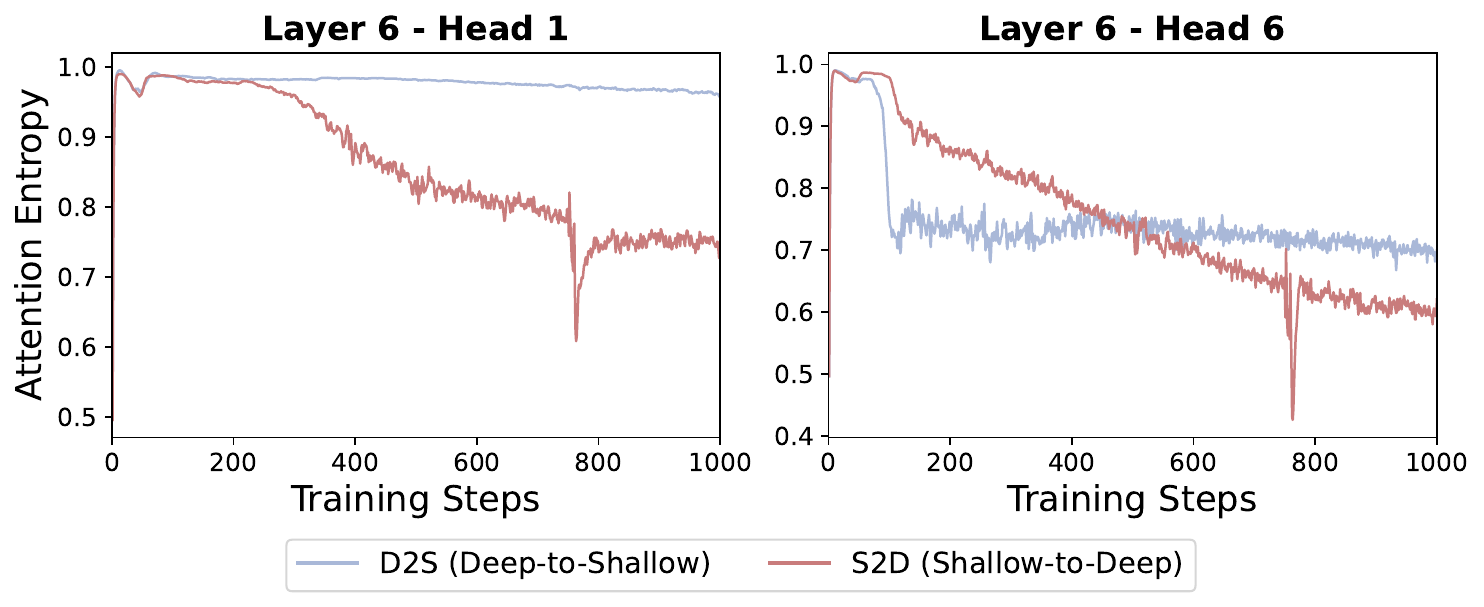} 
        \caption{Attention entropy dynamics of two heads in Layer 6 during training, including the warm-up phase and layer selection phase, for models trained with D2S and S2D strategies (more visualizations in Figure~\ref{fig: direction_stable}).}
    \label{fig: direction_stable_head1_head6}
\end{figure}

\begin{table}[!t]
\centering
\small
\renewcommand{\arraystretch}{1.2}
\setlength{\tabcolsep}{5.5pt}
\definecolor{darkgreen}{rgb}{0.0, 0.5, 0.0}
\definecolor{darkred}{rgb}{0.7, 0.0, 0.0}
\begin{tabular}{l|ccc}
\toprule
\textbf{Setting} & \textbf{Vanilla} & \textbf{Block Loop} & \textbf{High-Ent. Loop} \\
\midrule
1024$\times$1 & 24.21 & 24.18 (\textcolor{darkgreen}{$-0.03$}) & 24.16 (\textcolor{darkgreen}{$-0.05$}) \\
1024$\times$2 & 57.38 & 58.80 (\textcolor{darkred}{$+1.42$}) & 56.62 (\textcolor{darkgreen}{$-0.76$}) \\
1024$\times$3 & 116.94 & 119.48 (\textcolor{darkred}{$+2.54$}) & 111.66 (\textcolor{darkgreen}{$-5.28$}) \\
1024$\times$4 & 180.70 & 175.05 (\textcolor{darkgreen}{$-5.65$}) & 171.90 (\textcolor{darkgreen}{$-8.80$}) \\
\bottomrule
\end{tabular}
\caption{Long Context Extrapolation Results (PPL). \emph{Setting} denotes the sequence length for extrapolation evaluation. The configurations for High-Ent. Loop and Block Loop follow the \emph{Layer 2 settings} in Table~\ref{tab: ablation}.}
\label{tab: extrapolation}
\end{table}

\begin{figure*}[!t]
    \centering
    \includegraphics[width=\linewidth]{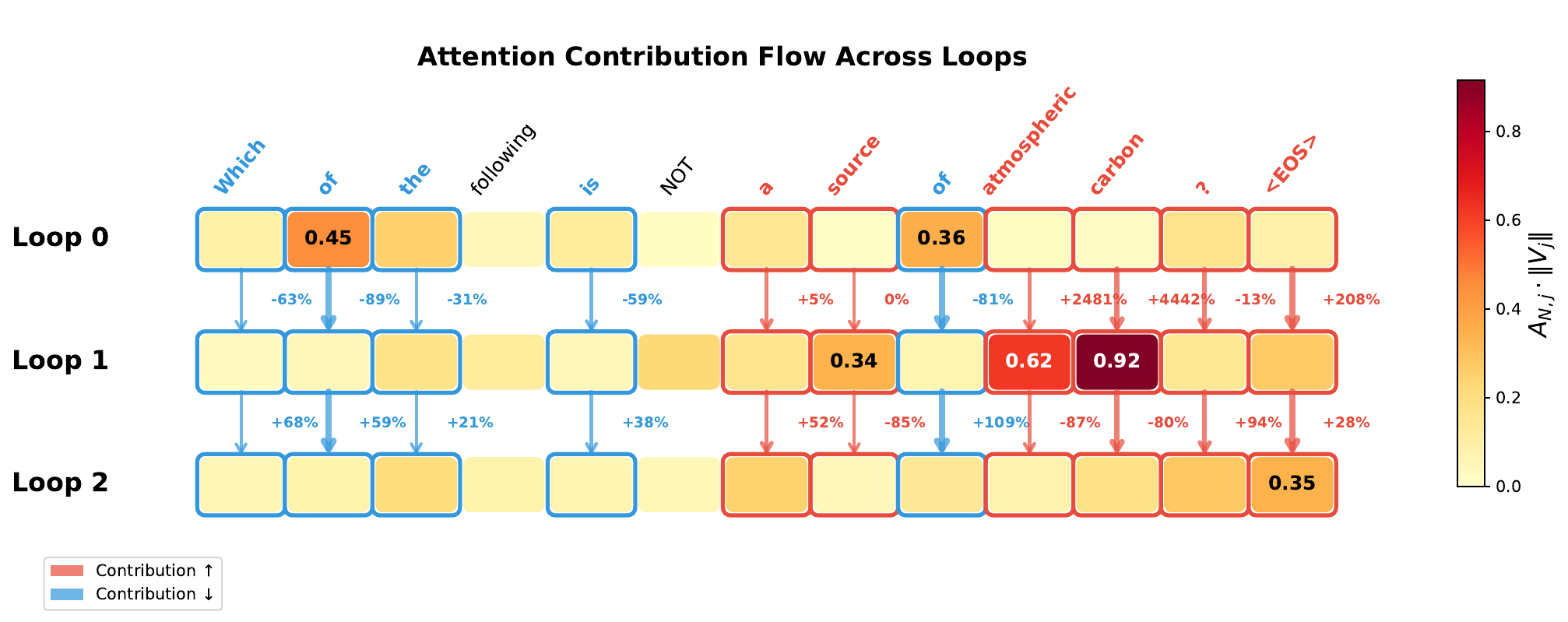} 
  \caption{Attention contribution flow across loops in a high-entropy head (Layer 2, Head 15). Each cell shows the weighted attention contribution $C_j = A_{N,j} \cdot \|V_j\|_2$ from the final token to position $j$, with color intensity indicating magnitude. Red and blue borders mark tokens with the largest contribution increase and decrease from Loop 0 to Loop 2, respectively. Arrows with percentage labels denote inter-loop changes. The visualization shows attention progressively shifting from syntactic elements toward task-critical and answer-relevant tokens.}
    \label{fig: attention_flow2}
\end{figure*}

\subsubsection{Impact of Growth Direction}

\paragraph{Performance Comparison.}
Guided by Observation~\ref{ob:maturation_trajectory}, which establishes that deeper layers differentiate and stabilize earlier, we design a Deep-to-Shallow (D2S) progressive growth strategy (Section~\ref{Progressive Growth Training}). To validate this design choice, we compare our proposed D2S approach against a reversed Shallow-to-Deep (S2D) variant on the 573M model with SGT ($K=1$) configuration. The only modification in the S2D variant is the reversal of the layer selection order during the progressive growth phase. As presented in Table~\ref{tab: growth_direction}, the D2S strategy consistently outperforms the S2D variant across all evaluation metrics. This empirical evidence validates our hypothesis that initiating structural growth from naturally stable deeper layers is crucial for maximizing the performance of the SGT.

\paragraph{Stability Analysis.}
To further investigate the underlying causes of this performance disparity, we analyze the attention entropy dynamics during the dynamic layer selection phase. We observe that the S2D strategy exhibits frequent entropy spikes, defined as abrupt fluctuations across multiple attention heads, indicating severe training instability. For instance, Layer 6 of the S2D model displays the most pronounced volatility (Figure~\ref{fig: direction_stable_head1_head6} and~\ref{fig: direction_stable}). In contrast, our D2S strategy maintains smooth entropy dynamics throughout the selection phase. This confirms that aligning the growth schedule with the model's intrinsic maturation trajectory ensures training stability, providing a robust foundation for progressive structural expansion.

\subsection{Analysis}

\paragraph{Allocating computational depth to high-entropy heads enhances long-context generalization.}
To further investigate generalization capabilities, we conduct long-context extrapolation experiments. We train models on sequences of length 1024 and evaluate them on progressively longer sequences (2048, 3072, and 4096 tokens), representing $2\times$, $3\times$, and $4\times$ extrapolation relative to training length. Crucially, these evaluations are conducted directly on extended sequences without utilizing any positional embedding extrapolation techniques (e.g., YARN~\cite{ICLR2024_874a4d89}, base-scaling), providing a strict test of the model's inherent extrapolation ability. As shown in Table~\ref{tab: extrapolation}, the High-Ent. Loop consistently outperforms the Vanilla baseline across all extrapolation settings, reducing PPL by 0.76, 5.28, and 8.80 points at $2\times$, $3\times$, and $4\times$ extrapolation, respectively. In comparison, the Block Loop exhibits degradation at moderate extrapolation levels ($2\times$ and $3\times$), suggesting that indiscriminate block-level recurrence introduces noise that interferes with length generalization.

\paragraph{High-Entropy Looping Enables Progressive Semantic Focusing.}
\label{sec:viz_analysis}
To qualitatively analyze the latent \textit{Inner Thinking} mechanism within high-entropy head looping, we visualize the attention dynamics of a head (Layer 2, Head 15) selected from SGT with $K=2$ loops. 
We quantify information flow using Weighted Attention Contribution $C_j = A_{N,j} \cdot \|V_j\|_2$, where $A_{N,j}$ denotes the attention weight from the final token to position $j$, and $\|V_j\|_2$ is the value vector norm. This metric captures each token's actual contribution to the output representation. 
Both Figure~\ref{fig: attention_flow2} and Figure~\ref{fig: attention_flow} reveal a consistent refinement process where the SGT iteratively focuses its semantic lens. 
Initially, Loop 0 exhibits scattered attention dispersed across both task-critical content and general syntactic elements. 
The first recurrence (Loop 1) triggers a sharp semantic filtering effect, drastically amplifying focus on pivotal tokens, such as \emph{risk} (+337\%) and \emph{individual} (+1080\%) in the medical context, and \emph{atmospheric} (+2481\%) and \emph{carbon} (+4442\%) in the scientific context, while simultaneously suppressing broad syntactic words like \emph{as}, \emph{of}, and \emph{the}.
By the second recurrence (Loop 2), the attention further converges toward answer-relevant positions to finalize the semantic anchoring. 
In this stage, key tokens like \emph{carbon} maintain high contributions while the model further eliminates residual attentional biases from the surrounding context. 
The transition from scattered attention to precise semantic anchoring demonstrates that SGT's recursive mechanism enables \textit{Inner Thinking} by iteratively correcting initial attentional biases and progressively refining semantic dependencies.

\section{Conclusion}
We establish attention entropy as a proxy for where added depth is valuable and uncover depth-dependent training dynamics. SGT then progressively loops high-entropy heads, which is theoretically shown to speed up convergence and empirically delivers stronger efficiency–performance and length generalization with interpretable inner thinking at minimal extra FLOPs.

\section*{Limitations}
The primary limitation of this work lies in the scale of the pre-training experiments. Due to computational resource constraints, our validation of SGT was conducted on models up to 1.2B parameters. While our analysis reveals consistent entropy patterns across the studied scales, confirming the scalability of the SGT framework to substantially larger foundation models remains an important direction for future work. In addition, we did not perform a fine-grained exploration of hyperparameter design, which may further improve the efficiency–performance trade-off.

\section*{Ethics Statement}
Our work studies architectural modifications for large language models using publicly available benchmarks. No personal or sensitive data is used. The proposed method does not introduce new deployment risks beyond those commonly associated with language models.

\section*{Acknowledgments}
This work is supported by the Beijing Nova Program (No. 20250484895).

\bibliography{custom}

\clearpage

\appendix

\section{Related Work}
\label{app: related}

\noindent\textbf{Depth Expansion and Recurrence.}
Increasing effective depth has been widely recognized as a key factor in enhancing Transformer reasoning capacity~\cite{petty2024impact,chen2024can,merrill2024little,saunshireasoning,roy2025dynamic,xue2025structcoh,zhang2026semanticawarelogicalreasoningsemiotic,xue2026supervisedfinetuningfailslearn,zhang2026logicalphasetransitionsunderstanding}. 
On the recurrence side, recent studies introduce recurrence to expand effective depth without increasing parameter count, showing improvements in implicit reasoning, input-length generalization, and training stability~\cite{ng2024loopneuralnetworksparameter,fan2024looped,geiping2025scaling,gong2025makes,wang2025think}.
These approaches primarily adopt block-level recursion, uniformly reapplying entire layers or layer groups to expand depth. 
While such recursive strategies achieve parameter efficiency by reusing existing weights, they introduce additional training-time computational overhead due to indiscriminate depth allocation across parameters, and leave unexplored which components within a block may benefit most from deeper recurrence.

\noindent\textbf{Paradigms of Sparsity.}
A prominent line of work improves efficiency through input-dependent sparsity, where computation is selectively allocated based on token characteristics~\cite{csordas2024switchhead,jin2024moh,wang2026fbs}. 
Mixture-of-Experts (MoE) architectures and head-level routing methods such as SwitchHead~\cite{csordas2024switchhead} and Mixture of Heads~\cite{jin2024moh} introduce input-dependent sparsity by routing tokens to distinct expert FFN parameters or selectively activating attention heads conditioned on token characteristics.
Similarly, recent inner-thinking dynamically selects which tokens receive additional computational steps~\cite{chen-etal-2025-inner}, while structured stepwise reasoning frameworks decompose complex tasks into explicit multi-stage cognitive processes~\cite{ji2026strideedstrategygroundedstepwisereasoning,zhang2025cooper,hao2026clear,yu2026knowrlboostingllmreasoning,xue2026reasonneededefficientgenerative}. 
These methods, therefore, introduce heterogeneous execution paths across samples, as computation varies with input-dependent routing decisions.
In contrast, our work introduces sparsity along a different axis. 
Rather than dynamically selecting tokens or experts, we focus on parameter-level depth allocation and training-time structural sparsity. 
Specifically, SGT progressively grows fine-grained recurrent micro-circuits during training, selectively allocating additional depth to high-value parameters within a block. 
This growth unfolds along the training timeline and concentrates computation on a subset of functional components, thereby reducing substantial pretraining overhead.
Since these approaches introduce sparsity along different dimensions, SGT is orthogonal to input-dependent routing methods and can be naturally combined with them to further improve efficiency.

\noindent\textbf{Attention Entropy.}
Attention entropy, defined as the Shannon entropy of attention weight distributions, provides a quantitative measure of how concentrated or dispersed a model's attention is across input tokens.
\citet{zhai2023stabilizing} demonstrates that fluctuations in average attention entropy correlate with training stability, while \citet{zhang2025attention} shows its influence on contextual modeling during reasoning.
However, these analyses typically operate at a coarse, model-wide granularity.
In the context of efficiency, HIES~\cite{choi2025entropy} combines gradient sensitivity with attention entropy to identify heads for pruning, reporting stable performance when removing high-entropy heads.
Nevertheless, its validation relies solely on simple sentiment classification.
For reasoning tasks characterized by complex dependencies, we posit that high-entropy heads may not be mere noise; instead, they may capture essential but widely distributed contextual signals~\cite{xue2023dual,he-etal-2025-breaking,ma2025semantic} required for global reasoning.

\begin{algorithm}[!t]
\caption{Progressive Growth Training}
\label{alg:SGT}
\small
\begin{algorithmic}[1]
\Require Model with $N_{\text{layer}}$ layers, training steps $T$, growth start $t_{\text{start}}$, window interval $\Delta t$, target number of looping layers $L$, the upper bound on the loop depth for the selected layers $K_{\max}$, the top-$L$ candidate layers per ranking forming the candidate set $\mathcal{C}$, and the top-$h$ high-entropy heads forming $\mathcal{S}_l$ for looping layer $l$.
\State Initialize $\mathcal{L} \leftarrow \emptyset$ (set of layers with active looping)
\State Initialize $K_l \leftarrow 0$ for all layers $l$ (current loop depth for layer $l$)
\State Initialize $l^* \leftarrow \text{None}$ (currently growing layer)

\For{$t = 1$ to $T$}
    \State \textcolor{gray}{// Forward pass with entropy-guided looping (Eq.~\ref{eq:sga_update})}
    \For{$l = 1$ to $N_{\text{layer}}$}
        \State $H^{(0)} \leftarrow H + \sum_{i} \operatorname{Attn}^{(i)}(H)$
        \If{$l \in \mathcal{L}$}
            \For{$k = 1$ to $K_l$}
                \State $H^{(k)} \leftarrow H^{(k-1)} + $
                \Statex $\qquad\qquad\qquad \sum_{i \in \mathcal{S}_l} \operatorname{Attn}^{(i)}(H^{(k-1)})$
            \EndFor
        \EndIf
        \State $H \leftarrow \operatorname{FFN}(H^{(K_l)})$
    \EndFor
    
    \State \textcolor{gray}{// Progressive growth (Selection Phase)}
    \If{$t \geq t_{\text{start}}$ and $(t - t_{\text{start}}) \mod \Delta t = 0$}
        \State $\mathcal{C} \leftarrow \operatorname*{arg\,top}_{L}\!\left(\{\bar{\mathcal{E}}^{(l)}\}_{l=1}^{N_{\text{layer}}}\right)$
        \If{$l^* \neq \text{None}$ and $l^* \in \mathcal{C}$ and $K_{l^*} < K_{\max}$}
            \State $K_{l^*} \leftarrow K_{l^*} + 1$
        \ElsIf{$|\mathcal{L}| < L$}
           \State $l^* \leftarrow \max \left\{ l \mid l \in \mathcal{C},\; l \notin \mathcal{L},\; l < \min(\mathcal{L}) \right\}$
            \State $\mathcal{S}_{l^*} \leftarrow \operatorname*{arg\,top}_{h}(\{\mathcal{E}^{(i)}\})$
            \State $\mathcal{L} \leftarrow \mathcal{L} \cup \{l^*\}$
            \State $K_{l^*} \leftarrow 1$
        \EndIf
    \EndIf
\EndFor
\end{algorithmic}
\end{algorithm}

\begin{table}[!t]
\centering
\setlength{\tabcolsep}{4pt}
\renewcommand{\arraystretch}{1.5}
\resizebox{0.98\columnwidth}{!}{%
\begin{tabular}{lc}
\toprule
\textbf{Benchmarks} & \textbf{Avg. Input Length} \\
\midrule
\multicolumn{2}{l}{\textbf{Mathematical Reasoning}} \\
\quad AIME~\cite{balunovic2025matharena} & 160 tokens \\
\quad MATH-500~\cite{lightman2023lets} & 89 tokens \\
\quad GSM8K~\cite{cobbe2021gsm8k} & 78 tokens \\
\midrule
\multicolumn{2}{l}{\textbf{General Knowledge QA}} \\
\quad MMLU~\cite{hendrycks2020measuring} & 117 tokens \\
\midrule
\multicolumn{2}{l}{\textbf{Code Reasoning}} \\
\quad HumanEval~\cite{chen2021evaluating} & 165 tokens \\
\midrule
\multicolumn{2}{l}{\textbf{Long-Context Retrieval}} \\
\quad RULER\_NIAH\_MK~\cite{hsieh2024ruler} & 3051 tokens \\
\quad RULER\_NIAH\_MV~\cite{hsieh2024ruler} & 3133 tokens \\
\midrule
\multicolumn{2}{l}{\textbf{Long-Context Multi-Hop QA}} \\
\quad NarrativeQA~\cite{kwiatkowski-etal-2019-natural} & 3473 tokens \\
\quad InfBenchQA~\cite{DBLP:journals/corr/abs-2402-13718} & 3760 tokens \\
\quad TriviaQA~\cite{joshi-etal-2017-triviaqa} & 3308 tokens \\
\quad HotpotQA~\cite{yang-etal-2018-hotpotqa} & 3287 tokens \\
\bottomrule
\end{tabular}%
}
\caption{Overview of the 11 diverse reasoning benchmarks, categorized by task family, used for analyzing the model's attention entropy. The tasks span various domains and sequence lengths.}
\label{tab: Overview of the 11 evaluation datasets}
\end{table}

\begin{table}[!t]
\centering
\setlength{\tabcolsep}{1.5pt}
\renewcommand{\arraystretch}{1.2}
\resizebox{0.99\columnwidth}{!}{%
\begin{tabular}{lcccc}
\toprule
\textbf{Model} & \textit{layers} &\textit{heads} &\textit{ hidden size} &\textit{intermediate size}  \\
\midrule
\textbf{Qwen3-0.6B-Base} &28 &16 &1024  &3072 \\
\midrule
\textbf{LLaMA3.2-1B-Base} &16 &32 &2048  &8192 \\
\midrule
\textbf{Qwen3-1.7B-Base} &28 &16 &2048  &6144 \\
\bottomrule
\end{tabular}%
}
\caption{Architectural hyperparameters of three representative open-source models used in our observational study of Section~\ref{Observations}. \textit{heads} denotes the number of query heads.}
\label{tab: Architectural hyperparameters}
\end{table}

\begin{figure*}[!t]
    \centering
    \includegraphics[width=\linewidth]{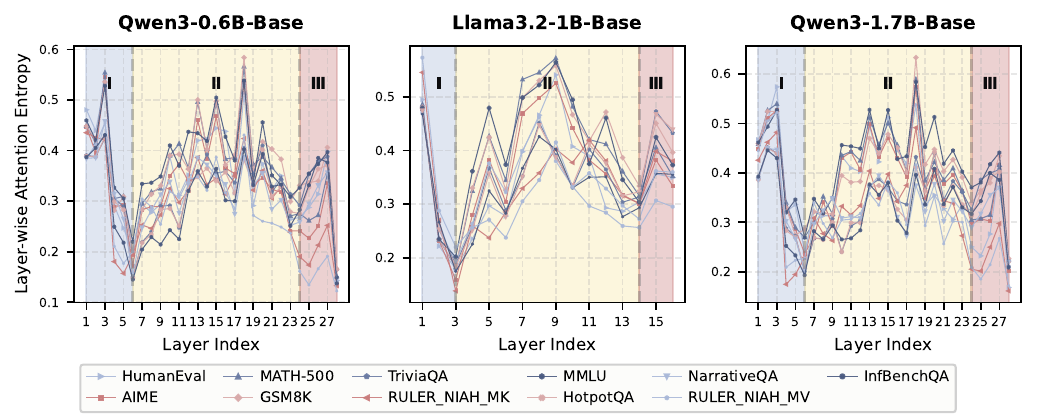} 
        \caption{Layer-wise attention entropy across 11 diverse tasks for three open-source models (details in Appendix~\ref{app: Experimental Setup} and ~\ref{app: Layer-wise Attention Entropy}).}
    \label{fig: layer_atten_entropy}
\end{figure*}

\section{Attention Entropy Computation and Rationale}
\label{app: Attention Entropy Computation and Rationale}
We utilize the attention distribution of the last token to compute the entropy metric for the following reasons.
First, the final query token is the position where the model integrates all preceding contextual information to produce the next-token distribution. Since next-token prediction is the core training objective of autoregressive LLMs, the last token of the input sequence inherently encodes all preceding information~\cite{muennighoff2022sgpt,wang2024improving,tang2024pooling,fu-etal-2025-token,DBLP:journals/corr/abs-2602-09953}, making its attention pattern the most semantically meaningful for characterizing how context is aggregated.
Second, prior work shows that in decoder-only Transformers, attention entropy tends to increase smoothly and monotonically along the sequence without abrupt fluctuations at intermediate positions~\cite{zhang2025attention}, making the entropy of the final token a stable summary of the model’s contextual uncertainty at the end of decoding. 
Third, this approach significantly reduces computational overhead. By avoiding full-sequence aggregation, we ensure the computational efficiency required for the effective implementation of the entropy metric in our proposed method.

\begin{figure*}[!t]
    \centering
    \includegraphics[width=\linewidth]{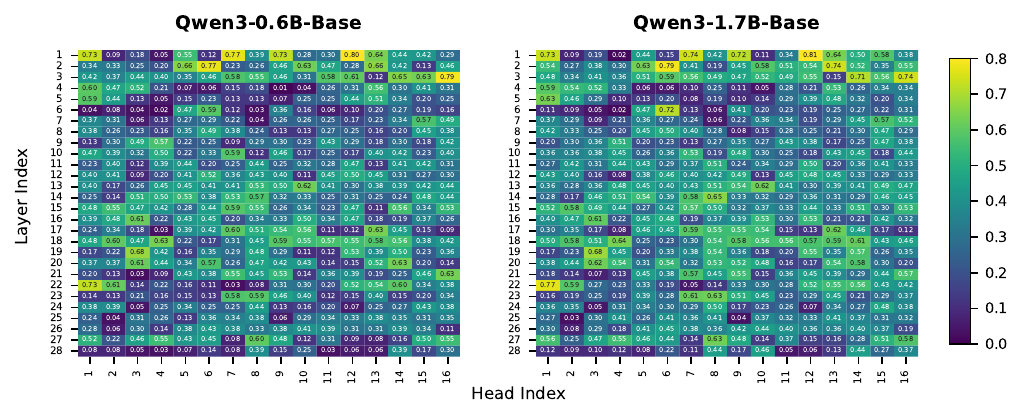}
        \caption{Layer–head attention entropy pattern of Qwen3-0.6B and Qwen3-1.7B. The x-axis denotes the head index, and the y-axis denotes the layer index (details in Appendix~\ref{app: Experimental Setup} and ~\ref{app: Layer-head Attention Entropy Pattern}).}
        \label{fig: Layer_Head_Attention_Entropy_Patterns}
\end{figure*}

\section{Observational Study Details}

\subsection{Experimental Setup}
\label{app: Experimental Setup}
We select three representative open-source models: Qwen3-0.6B-Base,  Qwen3.1-1.7B-Base~\cite{DBLP:journals/corr/abs-2505-09388}, and LLaMA3.2-1B-Base~\cite{MetaLlama3.2_2024}. 
These models span diverse parameter scales, and their architectural configurations are detailed in Table~\ref{tab: Architectural hyperparameters}.
To minimize task-specific bias, we evaluated the models across 11 diverse reasoning benchmarks. 
These tasks cover a broad range of domains and sequence lengths, including Mathematical Reasoning, General Knowledge QA, Code Reasoning, Long-Context Retrieval, and Long-Context Multi-Hop QA (see Table~\ref{tab: Overview of the 11 evaluation datasets}). 
Among these, all long-context tasks are drawn from HELMET~\cite{yen2025helmet}, a comprehensive benchmark for long-text reasoning ability. 
For each task, we randomly sample 100 instances; for datasets with many fine-grained subcategories, such as MMLU and MATH-500, we sample uniformly across all subtypes to reduce category bias.

\subsection{Intrinsic Structural Patterns of Attention Entropy}
\label{app: Intrinsic Structural Patterns of Attention Entropy}

\subsubsection{Layer-wise Attention Entropy}
\label{app: Layer-wise Attention Entropy}
We visualize the layer-wise attention entropy, as defined in Equation~\ref{eq:layer_mean_entropy}, across 11 diverse tasks on three representative open-source models.
The results (Figure~\ref{fig: layer_atten_entropy}) reveal a highly consistent and cross-task stable three-phase (I–II–III) entropy progression across all models: Phase I (shallow layers) shows a sharp entropy decrease; Phase II (middle layers) exhibits a gradual rise followed by a steady decline; Phase III (deep layers) presents a renewed entropy increase before a final decline before token generation. 
This robust and task-independent \textit{three-phase} progression demonstrates that attention entropy serves as an intrinsic architectural characteristic.
Notably, similar depth-dependent functional phases have previously been reported in interpretability studies~\cite{tenney2019bert,dai2022knowledge,fu2025cast,DBLP:conf/emnlp/ChenSZL25,huang2026semanticspaceexplorationexploitationrlvr}, but primarily through indirect diagnostics such as probing classifiers, neuron-level interventions, and spectral analyses of representation spaces. 
In contrast, layer-wise attention entropy provides a direct measurement of attention uncertainty and contextual focus, offering more task-agnostic and model-agnostic evidence for such depth-dependent functional specialization.

\subsubsection{Layer-head Attention Entropy Pattern}
\label{app: Layer-head Attention Entropy Pattern}
We visualize the layer-head attention entropy pattern, based on the mean attention entropy computed over all tasks.
As shown in Figure~\ref{fig: Layer_Head_Attention_Entropy_Patterns}, the two Qwen3 models with different parameter scales exhibit strikingly similar layer–head entropy patterns, despite substantial differences in their hidden sizes and intermediate sizes. 
Since the models share the same number of layers and attention heads (Table~\ref{tab: Architectural hyperparameters}), this consistency indicates that attention-entropy patterns are closely related to architectural configuration.
We also visualize the layer–head attention entropy patterns of LLaMA3.2-1B, as shown in Figure~\ref{fig: Appendix_Layer_Head_Attention_Entropy_Patterns}. 
Compared with the Qwen3 models, LLaMA3.2-1B exhibits a different layer–head entropy pattern.

\begin{figure}[!t]
    \centering
    \includegraphics[width=\linewidth]{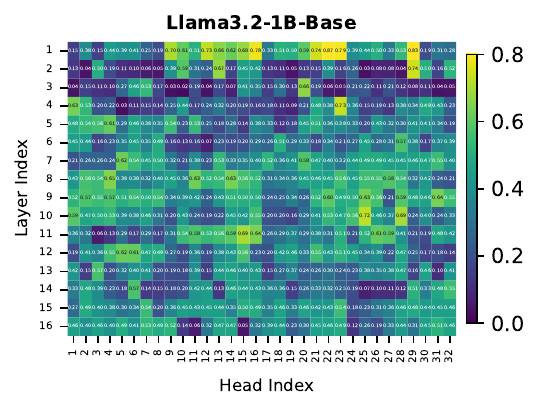}
        \caption{Layer–head attention entropy pattern of LLaMA3.2-1B-Base. The x-axis denotes the head index, and the y-axis denotes the layer index (details in Appendix~\ref{app: Experimental Setup} and ~\ref{app: Layer-head Attention Entropy Pattern}).}
        \label{fig: Appendix_Layer_Head_Attention_Entropy_Patterns}
\end{figure}

\subsection{Functional Dichotomy of High- and Low-Entropy Attention Heads}
\label{app: Functional Dichotomy of High- and Low-Entropy Attention Heads}

\subsubsection{Entropy-based Head Selection}
\label{app: Entropy-based Head Selection}
We identify high- and low-entropy heads in the Qwen3-0.6B model based on their attention entropy across 11 diverse tasks from Table~\ref{tab: Overview of the 11 evaluation datasets}. 
Heads whose entropy remains consistently above 0.5 across 11 tasks are classified as high-entropy heads, whereas those consistently below 0.06 across the same tasks are classified as low-entropy heads.
As shown in Table~\ref{tab: mask_results}, we use the notation Lx–Hy to denote the head at layer x and head index y (1-indexed).  
The identified high-entropy heads are L1-H1, L1-H7, L1-H12, and L2-H6, while the low-entropy heads are L4-H9, L25-H2, and L28-H4.
We analyze the functional roles of these two groups from both qualitative and quantitative perspectives.

\begin{table*}[!t]
\centering
\Large
\setlength{\tabcolsep}{6pt} 
\renewcommand{\arraystretch}{1.5} 
\resizebox{\textwidth}{!}{%
\begin{tabular}{l|c|c|c|cccc}
\toprule
\multirow{2}{*}{\textbf{Mask Settings}} & 
\cellcolor{lightred}\textbf{Original} & 
\textbf{High-Entropy} & 
\textbf{Low-Entropy} & 
\multicolumn{4}{c}{\textbf{Random Heads (Various Seeds)}} \\

 & \cellcolor{lightred}\textbf{Model}
 & \textbf{Heads} 
 & \textbf{Heads} 
 & \textbf{Set 1} & \textbf{Set 2} & \textbf{Set 3} & \textbf{Set 4} \\
\midrule

\textbf{Masked Heads} & 
\cellcolor{lightred}- & 
\makecell[c]{ L1-H1, L1-H7\\  L1-H12, L2-H6} & 
\makecell[c]{ L4-H9, L25-H2\\ L28-H4} & 
\makecell[c]{ L4-H1, L4-H5\\  L4-H12, L21-H8} & 
\makecell[c]{ L2-H5, L2-H6\\  L2-H8, L10-H4} & 
\makecell[c]{ L1-H9, L4-H8\\  L21-H8} & 
\makecell[c]{ L2-H5, L10-H15\\  L23-H12} \\
\midrule
\midrule

\textbf{MATH-500} & 
\cellcolor{lightred}27.0 & 
\makecell{5.0\\ \textcolor{reddrop}{\small $\downarrow 81.5\%$}} & 
\makecell{25.0\\ \textcolor{reddrop}{\small $\downarrow 7.4\%$}} & 
\makecell{20.0\\ \textcolor{reddrop}{\small $\downarrow 25.9\%$}} & 
\makecell{24.0\\ \textcolor{reddrop}{\small $\downarrow 11.1\%$}} & 
\makecell{25.0\\ \textcolor{reddrop}{\small $\downarrow 7.4\%$}} & 
\makecell{23.0\\ \textcolor{reddrop}{\small $\downarrow 14.8\%$}} \\
\midrule

\textbf{GSM8K} & 
\cellcolor{lightred}32.0 & 
\makecell{7.0\\ \textcolor{reddrop}{\small $\downarrow 78.1\%$}} & 
\makecell{29.0\\ \textcolor{reddrop}{\small $\downarrow 9.4\%$}} & 
\makecell{37.0\\ \textcolor{greenup}{\small $\uparrow 15.6\%$}} & 
\makecell{32.0\\ \textcolor{gray}{\small $\pm 0.0\%$}} & 
\makecell{32.0\\ \textcolor{gray}{\small $\pm 0.0\%$}} & 
\makecell{27.0\\ \textcolor{reddrop}{\small $\downarrow 15.6\%$}} \\
\midrule

\textbf{MMLU} & 
\cellcolor{lightred}50.5 & 
\makecell{18.0\\ \textcolor{reddrop}{\small $\downarrow 64.4\%$}} & 
\makecell{46.0\\ \textcolor{reddrop}{\small $\downarrow 8.9\%$}} & 
\makecell{45.5\\ \textcolor{reddrop}{\small $\downarrow 9.9\%$}} & 
\makecell{47.0\\ \textcolor{reddrop}{\small $\downarrow 6.9\%$}} & 
\makecell{46.0\\ \textcolor{reddrop}{\small $\downarrow 8.9\%$}} & 
\makecell{50.5\\ \textcolor{gray}{\small $\pm 0.0\%$}} \\
\midrule

\textbf{NarrativeQA} & 
\cellcolor{lightred}13.7 & 
\makecell{11.0\\ \textcolor{reddrop}{\small $\downarrow 19.7\%$}} & 
\makecell{14.2\\ \textcolor{greenup}{\small $\uparrow 3.7\%$}} & 
\makecell{13.6\\ \textcolor{reddrop}{\small $\downarrow 0.7\%$}} & 
\makecell{12.2\\ \textcolor{reddrop}{\small $\downarrow 11.0\%$}} & 
\makecell{12.6\\ \textcolor{reddrop}{\small $\downarrow 8.0\%$}} & 
\makecell{14.5\\ \textcolor{greenup}{\small $\uparrow 5.8\%$}} \\
\midrule

\textbf{NIAH\_MK} & 
\cellcolor{lightred}81.0 & 
\makecell{69.0\\ \textcolor{reddrop}{\small $\downarrow 14.8\%$}} & 
\makecell{80.0\\ \textcolor{reddrop}{\small $\downarrow 1.2\%$}} & 
\makecell{83.0\\ \textcolor{greenup}{\small $\uparrow 2.5\%$}} & 
\makecell{72.0\\ \textcolor{reddrop}{\small $\downarrow 11.1\%$}} & 
\makecell{74.0\\ \textcolor{reddrop}{\small $\downarrow 8.6\%$}} & 
\makecell{84.0\\ \textcolor{greenup}{\small $\uparrow 3.7\%$}} \\

\bottomrule
\end{tabular}%
}
\caption{Impact of head masking on Qwen3-0.6B-Base performance across reasoning and long-context tasks. The \emph{Original Model} column represents the unmasked model baseline. \emph{Mask Settings} compare the removal of High-Entropy heads against Low-Entropy and four groups of Random head baselines. The specific indices of removed heads are detailed in the \emph{Masked Heads} row, denoted as L\emph{x}-H\emph{y} (Layer \emph{x}, Head \emph{y}, 1-indexed). Cells report absolute performance scores and the relative percentage change compared to the \emph{Original Model} (details in Appendix~\ref{app: Entropy-based Head Selection} and~\ref{app: Quantitative Analysis}).}
\label{tab: mask_results}
\end{table*}

\subsubsection{Qualitative Analysis}
\label{app: Qualitative Analysis}
We visualize attention patterns on representative samples from three benchmarks: GSM8K (Figure~\ref{fig: sample_visual_gsm8k}), MATH-500 (Figure~\ref{fig: sample_visual_math500}), and MMLU (Figure~\ref{fig: sample_visual_mmlu_college_physics}), with consistent qualitative distinctions observed between high- and low-entropy heads across all datasets.
Along the horizontal axis, we highlight in red the tokens that receive the top 50\% of attention weights from the final query position.
While the top 50\% tokens of low-entropy heads remain confined to a narrow local neighborhood near the diagonal, high-entropy heads selectively attend to tokens that are directly relevant for reasoning, indicating their role in capturing globally meaningful dependencies.

\subsubsection{Quantitative Analysis}
\label{app: Quantitative Analysis}
We evaluate how masking different groups of attention heads affects model performance on the Qwen3-0.6B model across five tasks from Table~\ref{tab: Overview of the 11 evaluation datasets}, with all outputs assessed using GPT-4o-mini~\cite{openai2024gpt4omini} as a unified judge model.
To control for randomness, we construct four random-head baselines sampled from the remaining heads after excluding both high- and low-entropy heads: two groups match the high-entropy set in terms of layer and head count, and the other two match the low-entropy set.
As shown in Table~\ref{tab: mask_results}, masking high-entropy heads results in the greatest performance degradation relative to the original model across almost all tasks, with performance collapsing much more severely on reasoning datasets than on long-context retrieval tasks.

\begin{figure}[!t]
    \centering
    \includegraphics[width=\linewidth]{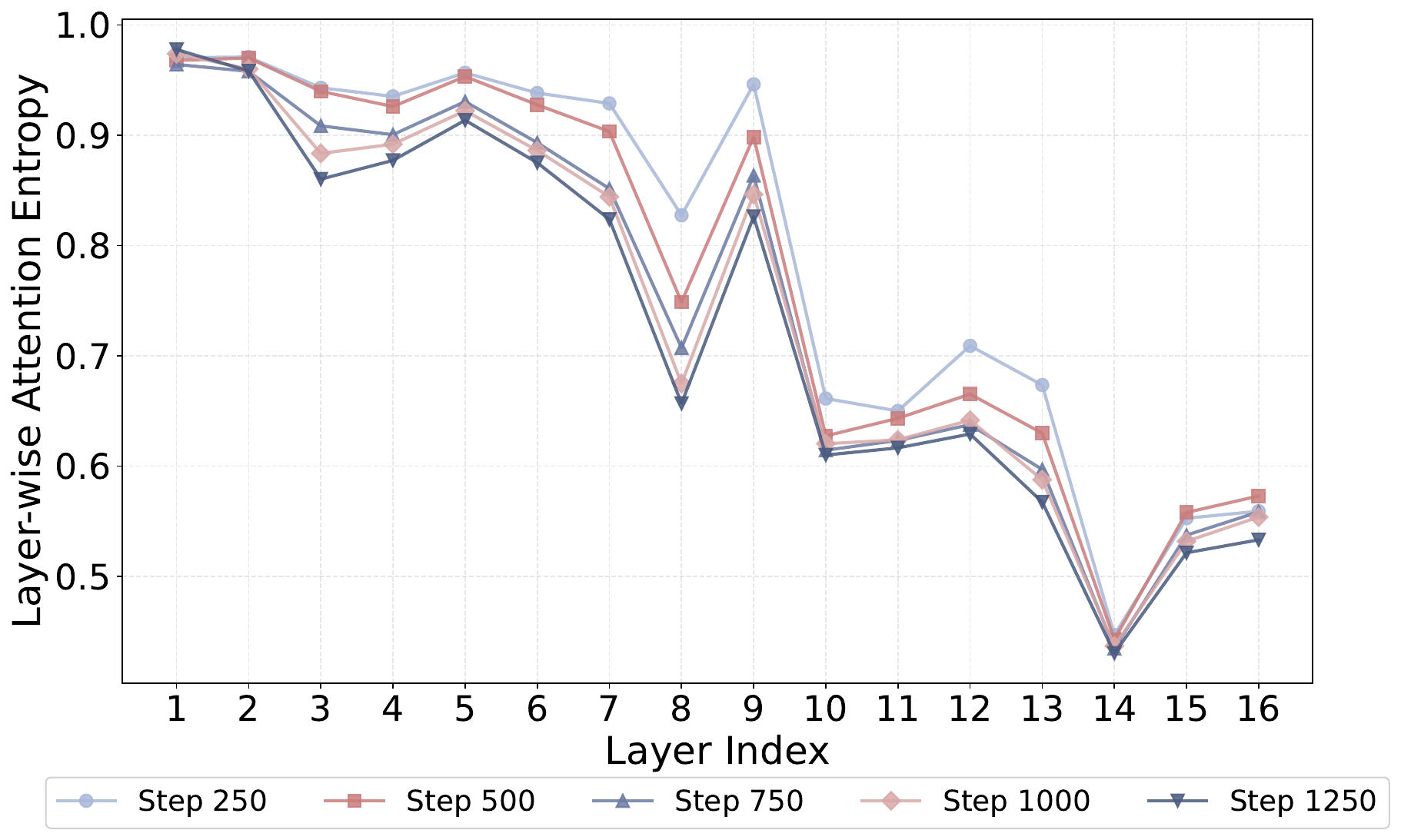} 
    \caption{Layer-wise Attention Entropy at Different Training Steps. The horizontal axis represents the model's layer index. Each line shows the entropy distribution recorded at 250-step intervals during the early stages of training (details in Appendix ~\ref{app: Training Dynamics of Attention Entropy}).}
    \label{fig: training/layer_mean_attention_entropy}
\end{figure}

\begin{table}[!t]
\centering
\setlength{\tabcolsep}{1.5pt}
\renewcommand{\arraystretch}{1.2}
\resizebox{0.99\columnwidth}{!}{%
\begin{tabular}{lccc}
\toprule
\textbf{Model Setting}  &\textbf{OLMo-275M} & \textbf{OLMo-573M} &\textbf{OLMo-1.2B}\\
\midrule
\textit{hidden size} &768  & 1024 &1536 \\
\textit{intermediate size} &6144  & 8192 &12288  \\
\textit{attention heads} &12  & 16  &16 \\
\textit{layers}  &12  &16  &16 \\
\bottomrule
\end{tabular}%
}
\caption{ Detailed architecture configurations for the three model scales (275M, 573M, and 1.2B) pre-trained from scratch in our main experiments. We pre-train models from scratch at three scales using the OLMo framework. All models adopt the LLaMA architecture with Multi-Head Attention (MHA).}
\label{tab: model_detailed_configuration}
\end{table}

\subsection{Training Dynamics of Attention Entropy}
\label{app: Training Dynamics of Attention Entropy}
We pre-trained a 575M parameter model from scratch based on the LLaMA architecture with Multi-Head Attention (MHA). 
The model comprises 16 layers with 16 attention heads per layer; detailed configurations are provided in Table~\ref{tab: model_detailed_configuration}. 
The pre-training was conducted on a corpus of 20B tokens for a total of 5,035 steps.
To investigate the dynamic properties of the model, we tracked three key metrics throughout the training process: the layer-wise average attention entropy (Equation~\ref{eq:layer_mean_entropy}), the individual attention entropy of all heads within each layer (Equation~\ref{eq: normalized-entropy}), and the intra-layer variance of head attention entropy.

Regarding the layer-level dynamics, we observe that during the early training phase, attention entropy decreases universally, while the relative ordering of layers remains largely stable with only minor fluctuations (Figure~\ref{fig: training/layer_mean_attention_entropy}). 

Regarding the intra-layer variance, we observe distinct behaviors across depths: deep layers exhibit a rapid surge followed by a gradual decline, while shallow layers show a steady, continuous increase.
Interpreting attention entropy as a proxy for head functionality (Observation~\ref{ob:functional_dichotomy}), we consider the intra-layer entropy variance as a measure of the degree of functional differentiation among heads. A higher variance indicates more significant divergence in head roles.
Based on this interpretation, the observed patterns reveal that deep layers undergo rapid functional differentiation during the early stages of training, whereas shallow layers differentiate through a slower, more stable process. Notably, we observe a clear depth-dependent trend: the deeper the layer, the earlier it achieves functional differentiation.
In particular, Layer 1 presents a notable exception: it maintains near-zero entropy variance throughout the entire training process, with the attention entropy of all heads consistently approaching 1 (Figure~\ref{fig: both}).
Given this lack of functional differentiation, we exclude Layer 1 from the candidate pool for our dynamic layer selection strategy during the progressive growth training phase.

\section{Experimental Details and Supplementary Analysis}
\label{app: Experimental Details and Supplementary Analysis}

\subsection{Details of Main Experiments}
\label{app: Details of Main Experiments}
The \emph{Block Loop} baseline follows a widely adopted block-level depth reuse design in prior looped-depth studies~\cite{DBLP:journals/corr/abs-2510-25741,chen-etal-2025-inner,DBLP:journals/corr/abs-2507-10524}. 
Its configuration is deliberately designed to ensure strong and competitive performance rather than serve as a weak reference point. 
In particular, the selection of looped layers is guided by attention entropy rather than random assignment, following the same principle used in our proposed method. 
This ensures that recurrence is applied to the most information-dense layers, maintaining methodological consistency and enabling a fair comparison between block-level and head-level depth allocation strategies.

Our primary objective is to validate the feasibility of the \emph{training-time sparsity paradigm}. 
We therefore do not attempt to reproduce the full complexity of existing token-routing loop architectures~\cite{chen-etal-2025-inner}, as discussed in Related Work, such methods explore a data-dependent execution paradigm that is orthogonal to our structural design focus. 
For the selective head-level looping in SGT, we set the number of high-entropy heads to $h=2$, which represents an extremely small fraction of the total attention heads within a layer. 
This deliberately sparse configuration allows us to evaluate whether substantial reductions in training FLOPs, achieved by looping only a minimal subset of high-entropy heads, can match the performance of full block-level looping. 
All results are reported from single runs, which is standard practice in large-scale pretraining due to computational constraints.

\subsection{Details of Ablation Experiments}
\label{app: Details of Ablation Experiments}
To isolate individual variables and minimize confounding effects in our ablation and analysis experiments, we adopted a single-layer configuration. This design enables a clean assessment of the computational contributions of different block components across varying loop depths. For the high-entropy head loop, we used a two-stage selection procedure: in the first window, we identified ten high-entropy heads; in the second window, we selected two heads from this set. The resulting structure was then frozen for the remainder of training. The low-entropy head loop follows the same protocol, with the only difference being that low-entropy heads are selected.

\subsection{Analysis of the Growth Schedule Parameters}
\label{app: Analysis of the Growth Schedule Parameters}

\begin{table}[!t]
\centering
\setlength{\tabcolsep}{4pt}
\renewcommand{\arraystretch}{1}
\resizebox{0.93\columnwidth}{!}{%
\begin{tabular}{lccc}
\toprule
\textbf{Model} & \textbf{ARC-E} & \textbf{WG} & \textbf{SIQA} \\
\midrule
Vanilla & 44.56 & 49.57 & 41.30 \\
Block Loop & 44.74 & 50.36 & 41.10 \\
SGT ($t_{\text{start}}=\Delta t=500$) & 45.26 & 50.91 & 41.56 \\
SGT ($t_{\text{start}}=\Delta t=250$) & 45.09 & 51.46 & 41.40 \\
SGT ($t_{\text{start}}=\Delta t=100$) & 45.44 & 51.62 & 41.97 \\
\bottomrule
\end{tabular}%
}
\caption{Analysis of growth schedule parameters ($t_{\text{start}}$ and $\Delta t$) for the 275M model.}
\label{tab: schedule_ablation}
\end{table}

In our main experiments, we configure the progressive growth schedule with 
$t_{\text{start}} = 250$ and $\Delta t = 250$. 
To further examine the sensitivity of SGT to the growth schedule, we conduct additional experiments by varying the temporal parameters $t_{\text{start}}$ and $\Delta t$. 
As shown in Table~\ref{tab: schedule_ablation}, SGT consistently achieves stable and robust performance improvements across different configurations. 

\subsection{Analysis of Inference Throughput}
\label{app: Analysis of Inference Throughput}

\begin{table}[!t]
\centering
\setlength{\tabcolsep}{2pt}
\renewcommand{\arraystretch}{1.1}
\resizebox{0.99\columnwidth}{!}{%
\begin{tabular}{c| l c c}
\toprule
\textbf{Batch} & \textbf{Model} & \textbf{Prefill TPS} & \textbf{Decode TPS} \\
\midrule

\multirow{6}{*}{1}
& Vanilla (275M) & 14381.0 & 127.9 \\
& SGT (275M) & 13805.7 & 122.7 \\
& Block Loop (275M) & 11504.8 & 102.3 \\
& Vanilla (573M) & 10969.2 & 100.5 \\
& SGT (573M) & 10718.0 & 98.2 \\
& Block Loop (573M) & 9237.2 & 84.6 \\
\midrule

\multirow{6}{*}{2}
& Vanilla (275M) & 27646.2 & 245.2 \\
& SGT (275M) & 26540.4 & 235.3 \\
& Block Loop (275M) & 22117.0 & 196.1 \\
& Vanilla (573M) & 3997.6 & 194.9 \\
& SGT (573M) & 3906.1 & 190.4 \\
& Block Loop (573M) & 3366.4 & 164.1 \\
\midrule

\multirow{6}{*}{4}
& Vanilla (275M) & 98181.9 & 928.5 \\
& SGT (275M) & 94254.6 & 891.4 \\
& Block Loop (275M) & 78545.5 & 742.8 \\
& Vanilla (573M) & 41193.5 & 384.6 \\
& SGT (573M) & 40250.1 & 375.8 \\
& Block Loop (573M) & 34689.2 & 323.9 \\
\bottomrule
\end{tabular}%
}
\caption{Inference throughput comparison (tokens/s) under different batch sizes.
Prefill throughput measures tokens processed during context encoding (KV cache construction), while decode throughput measures autoregressive token generation speed.}
\label{tab: throughput}
\end{table}

During inference, SGT assigns additional recurrence only to a small subset of attention heads identified during training as high-entropy components. 
Unlike data-dependent dynamic routing architectures that perform token-level or sample-level conditional execution, SGT maintains a static inference graph in which all samples within a batch follow the same computation path. 
This design preserves dense and regular execution, enabling full GPU parallelism without introducing branch divergence or dynamic control overhead.

We evaluate inference throughput on a single NVIDIA A800 GPU using BF16 precision, with a maximum prompt length of 128 and a generation length of 128. 
To isolate architectural effects, all methods are implemented natively without specialized operator-level optimizations. 
As shown in Table~\ref{tab: throughput}, although parameter reuse inevitably introduces moderate latency compared to the vanilla Transformer, SGT consistently achieves higher throughput than the Block Loop across all batch sizes and model scales. 

\section{Theoretical Analysis of High-Entropy Recursive Convergence}
\label{app: Theoretical Analysis}

In this section, we provide a theoretical justification for why high-entropy heads exhibit superior convergence properties compared to low-entropy heads in recursive attention mechanisms. 
Our analysis connects the entropy of attention matrices to the spectral properties of the recursive dynamics, offering theoretical insight into the experimental results.
We retain the notation from Section~\ref{Observations} and Section~\ref{Method} for consistency, while defining new variables as they appear.

\begin{figure*}[!t]
    \centering
    \includegraphics[width=\linewidth]{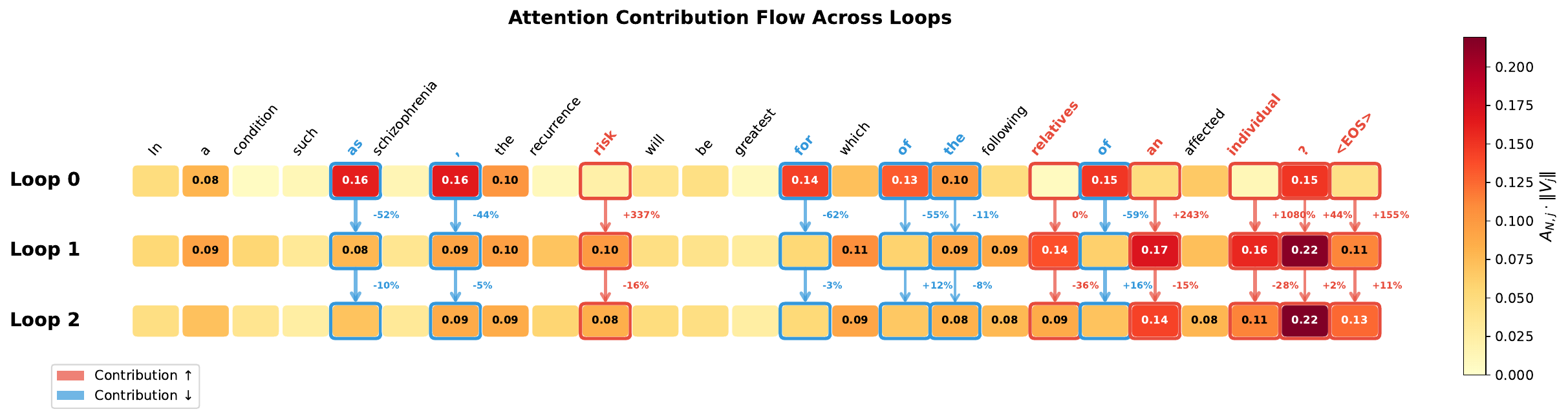}
        \caption{Attention contribution flow across loops in a high-entropy head (Layer 2, Head 15). Each cell shows the weighted attention contribution $C_j = A_{N,j} \cdot \|V_j\|_2$ from the final token to position $j$, with color intensity indicating magnitude. Red and blue borders mark tokens with the largest contribution increase and decrease from Loop 0 to Loop 2, respectively. Arrows with percentage labels denote inter-loop changes. The visualization shows attention progressively shifting from syntactic elements toward task-critical and answer-relevant tokens.}
        \label{fig: attention_flow}
\end{figure*}

\subsection{Convergence Objective and Equilibrium Analysis}
\label{sec:equilibrium_objective}

While standard Transformers stack distinct layers to refine token representations progressively, our recursive approach performs this refinement within a shared parameter space. 
To analyze the efficiency of this refinement, we view the recursive dynamics through the lens of implicit depth models.

\paragraph{Quasi-Stationary State.}
Following the framework of \cite{bai2019deep,fung2021jfbjacobianfreebackpropagationimplicit}, we analyze convergence toward a \emph{quasi-stationary} state $H^*$ where the recursive update becomes sufficiently small:
\begin{equation}
\|\mathcal{F}(H^*)\|_F \leq \delta,
\end{equation}
for some small tolerance $\delta > 0$. 
This formulation allows us to analyze finite-depth recursive dynamics without requiring exact convergence.

\paragraph{Convergence Error Analysis.}
We define the error at iteration $k$ as $\epsilon^{(k)} = H^{(k)} - H^*$. Our theoretical goal is to establish that after a finite number of recursive steps $K$, high-entropy head selection achieves a smaller error bound $\|\epsilon^{(K)}\|_F$ compared to low-entropy selections, thereby more efficiently approximating the quasi-equilibrium.

\subsection{Spectral Analysis of Convergence}
The key insight underlying our analysis is that global context integration is fundamentally a \emph{token-mixing} problem. 
The challenge in recursive refinement is achieving consensus across spatially distributed tokens—this information diffusion process is the primary bottleneck for convergence.
Recent spectral analyses of attention mechanisms \cite{pmlr-v267-nait-saada25a} 
have demonstrated that the spectral properties of the attention matrix $A$ 
(particularly the distribution of its eigenvalues) dominantly govern signal 
propagation dynamics through transformer layers.
The connectivity structure encoded in $A$ determines how quickly information flows between tokens, largely independent of the specific feature-space transformations applied by $W_{V}$ and $W_{O}$.

\paragraph{Spectral Dominance Approximation.}
Building on this spectral perspective, we model the aggregate effect of selected attention heads through their token-mixing properties. The aggregate recursive update can be approximated by factoring out an effective scaling factor $\gamma$ that absorbs the contributions of the projection matrices:
\begin{equation}
\mathcal{F}(H^{(k)}) = \sum_{i \in \mathcal{S}} \operatorname{Attn}^{(i)}(H^{(k)}) \approx \gamma \bar{A} H^{(k)},
\label{eq: Spectral Dominance Approximation}
\end{equation}
where $\bar{A}$ is the effective average attention matrix defined as:
\begin{equation}
\bar{A} = \frac{1}{|\mathcal{S}|} \sum_{i \in \mathcal{S}} A^{(i)}.
\end{equation}
This factorization separates the token-mixing dynamics (captured by $\bar{A}$) from feature-space transformations (absorbed into $\gamma$), allowing us to focus our spectral analysis on the attention matrix, which governs the fundamental convergence rate.

\paragraph{Error Dynamics.}
We analyze how the error $\epsilon^{(k)}$ evolves over iterations. 
Starting from the recursive update $H^{(k+1)} = H^{(k)} + \mathcal{F}(H^{(k)})$ and the quasi-equilibrium $H^* \approx H^* + \mathcal{F}(H^*)$, we obtain:
\begin{equation}
\begin{aligned}
\epsilon^{(k+1)} &= H^{(k+1)} - H^* \\
&\approx \epsilon^{(k)} + \mathcal{F}(H^{(k)}) - \mathcal{F}(H^*).
\end{aligned}
\end{equation}
Substituting the spectral approximation from Equation~\ref{eq: Spectral Dominance Approximation}:
\begin{equation}
\begin{aligned}
\epsilon^{(k+1)} &\approx \epsilon^{(k)} + \gamma \bar{A} H^{(k)} - \gamma \bar{A} H^* \\
&= \epsilon^{(k)} + \gamma \bar{A} (H^{(k)} - H^*) \\
&= (I + \gamma \bar{A}) \epsilon^{(k)}.
\end{aligned}
\end{equation}
To model the effective mixing strength (accounting for LayerNorm stabilization and residual connection balancing), we reparameterize with a normalized coefficient $\beta = \gamma/(1+\gamma) \in (0, 1]$, yielding:
\begin{equation}
\begin{aligned}
\epsilon^{(k+1)} &\approx ((1-\beta)I + \beta\bar{A}) \epsilon^{(k)} \\
&= M \epsilon^{(k)},
\end{aligned}
\end{equation}
where $M = (1-\beta)I + \beta\bar{A}$ is the effective mixing matrix.
Since $\bar{A}$ is row-stochastic (each row sums to 1), $M$ is also row-stochastic

\paragraph{Spectral Characterization of Convergence Rate.}
The convergence of the error $\epsilon^{(k)} = H^{(k)} - H^*$ is governed by the spectral properties of the mixing matrix $M$. For the error dynamics $\epsilon^{(k+1)} = M\epsilon^{(k)}$, we can bound the error norm by analyzing the eigenvalue structure of $M$.
Since $M = (1-\beta)I + \beta\bar{A}$, the eigenvalues of $M$ are related to those of $\bar{A}$ by:
\begin{equation}
\lambda_i(M) = (1-\beta) + \beta\lambda_i(\bar{A}), \quad i = 1, \ldots, N.
\end{equation}
Due to the causal mask in decoder-only architectures, $\bar{A}$ is strictly lower triangular, and thus its eigenvalues equal its diagonal entries: $\lambda_i(\bar{A}) = \bar{A}_{ii}$. 
For a fixed number of recursive steps $K$, the error after $K$ iterations can be bounded using the spectral norm:
\begin{equation}
\begin{aligned}
\|\epsilon^{(K)}\|_F &\leq \|M^K\|_F \|\epsilon^{(0)}\|_F \\
&\lesssim \left(\sum_{i=1}^{N} \lambda_i(M)^2\right)^{K/2} \|\epsilon^{(0)}\|_F \\
&= \left(\sum_{i=1}^{N} [(1-\beta) + \beta\bar{A}_{ii}]^2\right)^{K/2} \|\epsilon^{(0)}\|_F.
\end{aligned}
\end{equation}
Expanding the squared terms and noting that the $(1-\beta)^2$ contributes uniformly across all eigenvalues, the variation in convergence behavior is determined by the attention-dependent terms. 
The dominant factor for comparing different head selections is the trace:
\begin{equation}
\left(\sum_{i=1}^{N} \bar{A}_{ii}\right)^{K} = (\operatorname{Tr}(\bar{A}))^K,
\end{equation}
where $\operatorname{Tr}(\bar{A}) = \sum_{i=1}^{N} \bar{A}_{ii}$ is the trace of the attention matrix. 
This reveals that minimizing the trace of the attention matrix accelerates convergence. 
Our theoretical goal is to show that high-entropy head selection achieves this.

\subsection{Connecting Entropy to Trace}

We now establish a quantitative connection between attention entropy and the trace $\operatorname{Tr}(\bar{A})$.

\begin{lemma}[Entropy-Based Bound on Diagonal Elements]
\label{lemma:entropy_diagonal}
For the $i$-th row of a causal attention matrix $\bar{A}$ with entropy $\mathcal{E}_i = -\sum_{j=1}^{i} \bar{A}_{ij} \log \bar{A}_{ij}$, the diagonal element satisfies:
\begin{equation}
\bar{A}_{ii} \leq e^{-\mathcal{E}_i/i}.
\label{eq: entropy_bound}
\end{equation}
\end{lemma}

\begin{proof}[Proof Sketch]
Consider the $i$-th row as a probability distribution over positions $\{1, \ldots, i\}$. We seek to maximize $\bar{A}_{ii}$ subject to the constraints $\sum_{j=1}^{i} \bar{A}_{ij} = 1$ and $-\sum_{j=1}^{i} \bar{A}_{ij} \log \bar{A}_{ij} = \mathcal{E}_i$.
Applying the method of Lagrange multipliers, the optimal configuration that maximizes the diagonal weight while maintaining fixed entropy concentrates mass on the diagonal and distributes the remaining mass uniformly over other positions. Standard variational calculus yields that the maximum achievable diagonal weight decays exponentially with entropy.
\renewcommand{\qedsymbol}{}
\end{proof}

\paragraph{Bounding the Trace.}
Applying Lemma~\ref{lemma:entropy_diagonal} to each row:
\begin{equation}
\operatorname{Tr}(\bar{A}) = \sum_{i=1}^{N} \bar{A}_{ii} \leq \sum_{i=1}^{N} e^{-\mathcal{E}_i/i}.
\end{equation}
By first-order Taylor expansion around the mean attention entropy $\bar{\mathcal{E}} = \frac{1}{N}\sum_{i=1}^{N}\mathcal{E}_i$, the sum can be approximated as:
\begin{equation}
\sum_{i=1}^{N} e^{-\mathcal{E}_i/i} \approx \sum_{i=1}^{N} e^{-\bar{\mathcal{E}}/N} = N \cdot e^{-\bar{\mathcal{E}}/N}.
\end{equation}
Therefore, we obtain:
\begin{equation}
\operatorname{Tr}(\bar{A}) \lesssim N \cdot e^{-\bar{\mathcal{E}}/N},
\end{equation}
where the attention entropy $\bar{\mathcal{E}}$ directly controls the convergence rate: higher entropy yields a tighter bound on $\operatorname{Tr}(\bar{A})$, leading to accelerated decay of the error $\|\epsilon^{(K)}\|_F$ over $K$ recursive steps.
Our theoretical analysis establishes that high entropy of the attention matrix accelerated convergence over finite recursive steps. 
This provides the theoretical foundation for entropy-based head selection strategies.

\begin{figure*}[htbp]
    \centering
    \includegraphics[
        width=0.98\linewidth,
        height=0.98\textheight,
        keepaspectratio
    ]{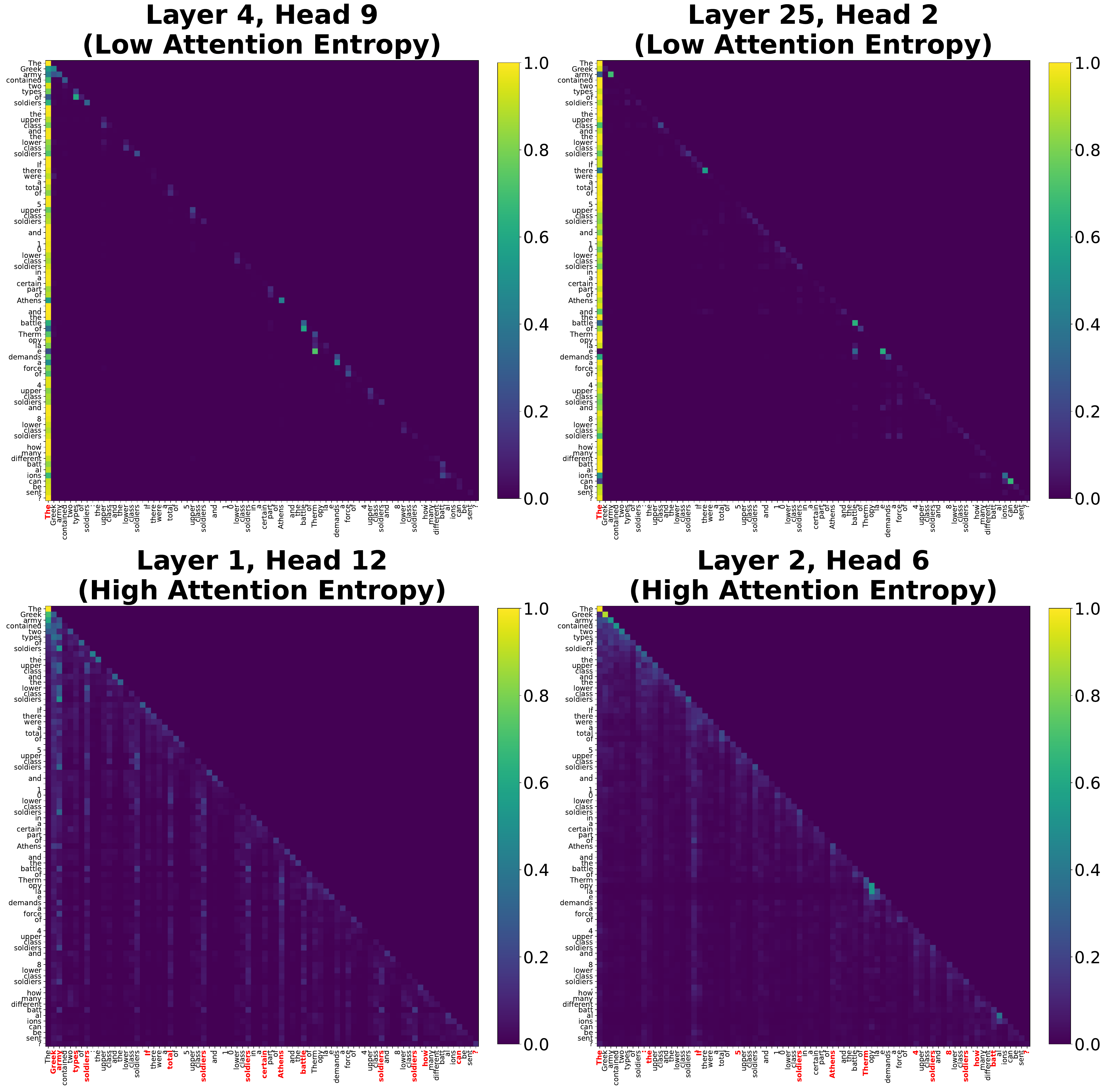}
    \caption{Visualization of attention heatmap in low- and high-entropy heads from Qwen3-0.6B, with a sample from MATH-500. Along the horizontal axis, tokens highlighted in red denote the subset that receives the top 50\% of the attention from the final query position (details in Appendix~\ref{app: Entropy-based Head Selection} and ~\ref{app: Qualitative Analysis}).}
    \label{fig: sample_visual_math500}
\end{figure*}

\begin{figure*}[htbp]
    \centering
    \includegraphics[
        width=0.98\linewidth,
        height=0.98\textheight,
        keepaspectratio
    ]{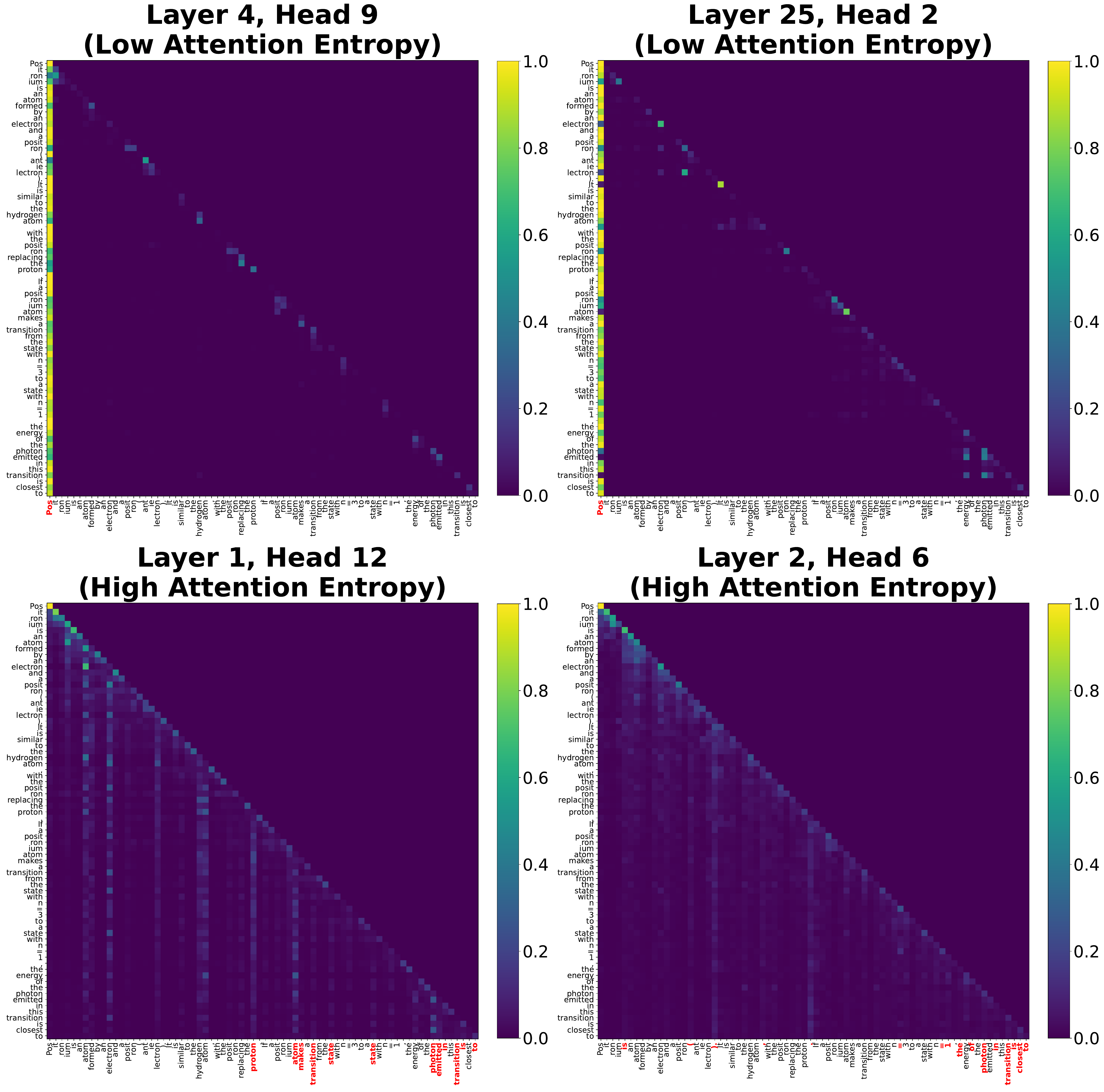}
    \caption{Visualization of attention heatmap in low- and high-entropy heads from Qwen3-0.6B, with a sample from MMLU. Along the horizontal axis, tokens highlighted in red denote the subset that receives the top 50\% of the attention from the final query position (details in Appendix~\ref{app: Entropy-based Head Selection} and ~\ref{app: Qualitative Analysis}).}
    \label{fig: sample_visual_mmlu_college_physics}
\end{figure*}

\begin{figure*}[htbp]
    \centering
    \begin{subfigure}[b]{0.87\textwidth}
        \centering
        \includegraphics[width=\linewidth]{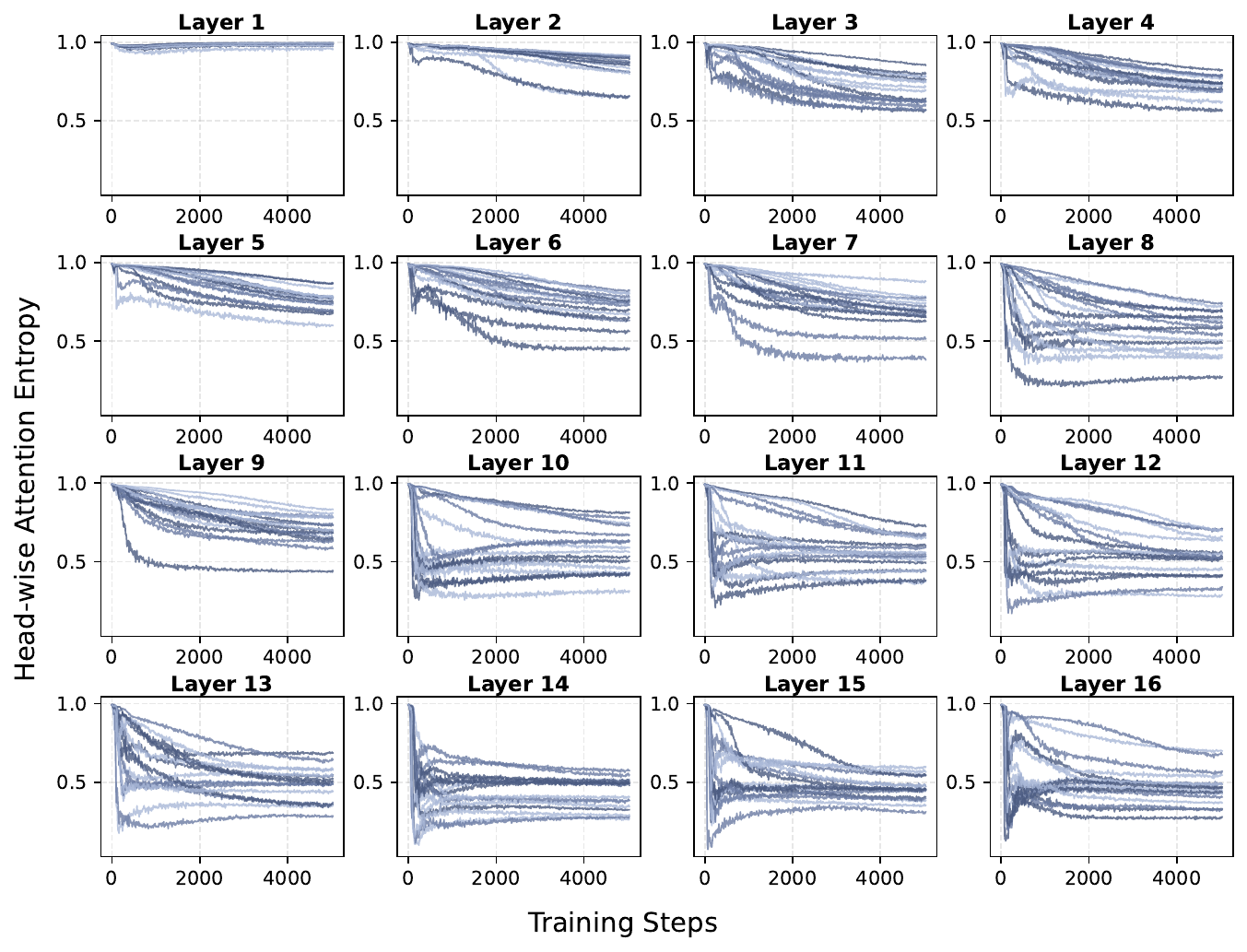}
        \caption{}
        \label{fig: training/head_entropy}
    \end{subfigure}
    \vspace{0.5em}
    \begin{subfigure}[b]{0.87\textwidth}
        \centering
        \includegraphics[width=\linewidth]{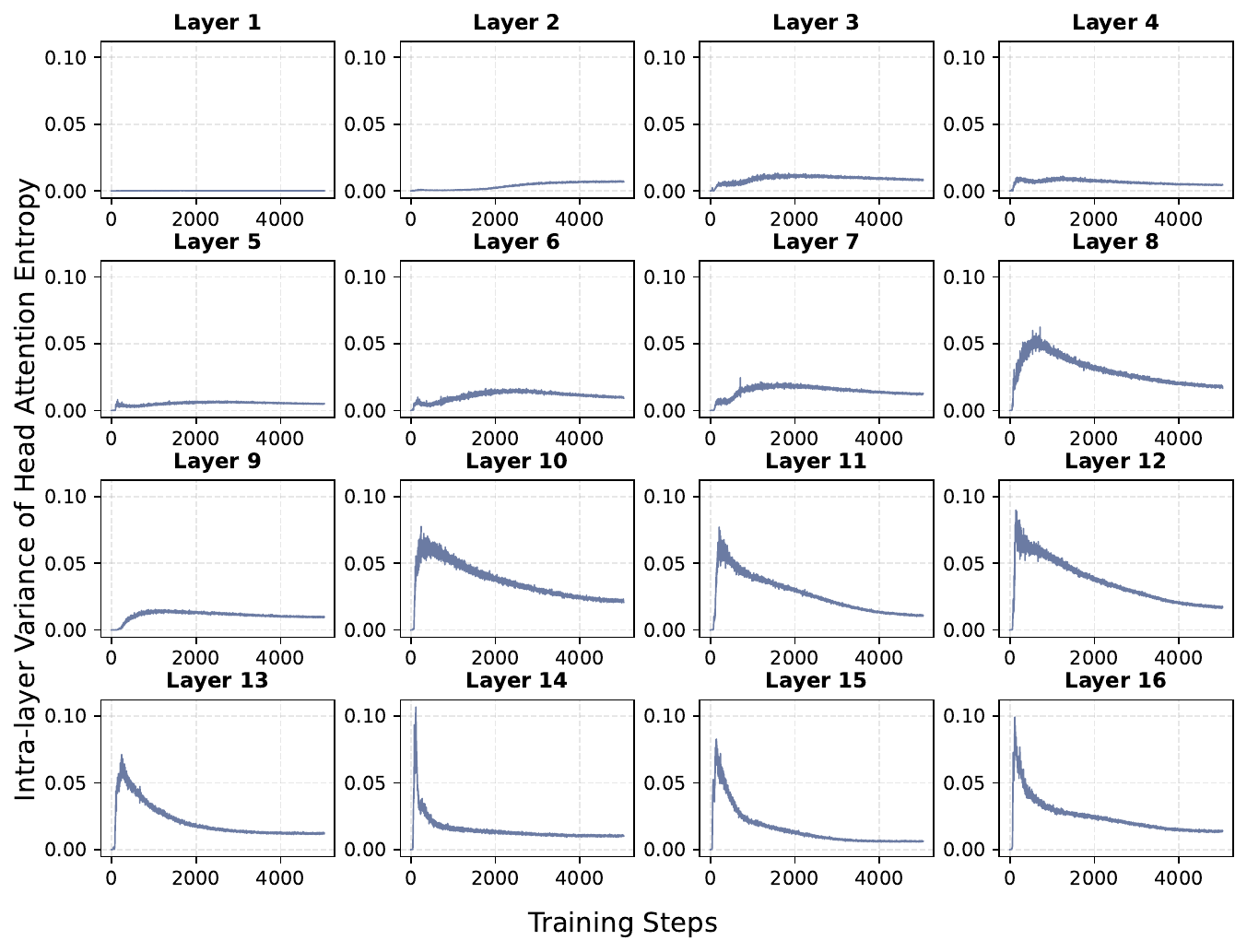}
        \caption{}
        \label{fig: training/head_variance_unified}
    \end{subfigure}

    \caption{(a) \textbf{Head-wise attention entropy trajectories across all layers during training}. Within each layer-specific subplot, individual curves correspond to distinct attention heads. (b) \textbf{Intra-layer variance of attention entropy across heads over the course of training (details in Appendix~\ref{app: Training Dynamics of Attention Entropy})}.}
    \label{fig: both}
\end{figure*}

\begin{figure*}[!tb]
    \centering
    \includegraphics[width=\linewidth]{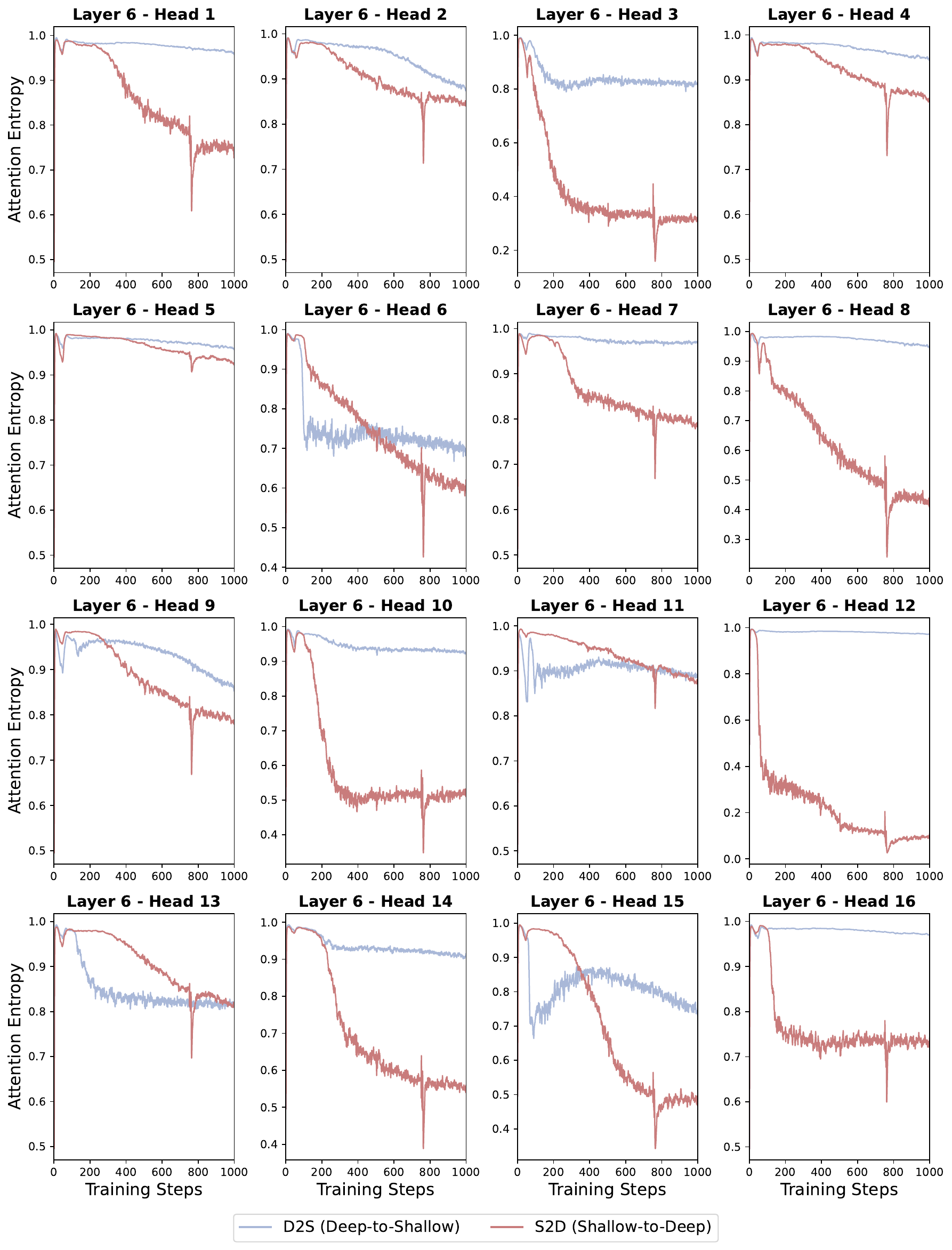} 
        \caption{Attention entropy dynamics of Layer 6 heads during training, including the warm-up phase (steps 0--250) and layer selection phase (steps 250--1000), for models trained with D2S and S2D strategies in Table~\ref{tab: growth_direction}.}
    \label{fig: direction_stable}
\end{figure*}

\end{document}